\documentclass{ecai}
\usepackage{times}
\usepackage{graphicx}
\usepackage{latexsym}
\usepackage{tikz-network}

\def\year{2020}\relax
\usepackage{amsmath}
\usepackage{amssymb}
\usepackage{amsthm}
\usepackage{tabularx}
\usepackage[most]{tcolorbox}
\usepackage{array}
\usepackage{blkarray}
\usepackage{bm}
\usepackage[utf8]{inputenc}
\usepackage{microtype}

\newtheorem{heuristic}{Heuristic}
\newtheorem{theoremdef}{Theorem}

\newcommand{\ER}{Erd\H os-R\'enyi}

\title{Heuristics for Link Prediction in Multiplex Networks}
\author{Robert E. Tillman \and Vamsi K. Potluru \and Jiahao Chen \and Prashant Reddy \and Manuela Veloso\institute{JPMorgan AI Research, email: robert.e.tillman@jpmorgan.com} }

\begin{document}

\maketitle

\begin{abstract}
\emph{Link prediction}, or the inference of future or missing connections between entities, is a well-studied problem in network analysis. A multitude of heuristics exist for link prediction in ordinary networks with a single type of connection. However, link prediction in \emph{multiplex networks}, or networks with multiple types of connections, is not a well understood problem. We propose a novel general framework and three families of heuristics for multiplex network link prediction that are simple, interpretable, and take advantage of the rich connection type correlation structure that exists in many real world networks. We further derive a theoretical threshold for determining when to use a different connection type based on the number of links that overlap with an \ER{} random graph. Through experiments with simulated and real world scientific collaboration, transportation and global trade networks, we demonstrate that the proposed heuristics show increased performance with the richness of connection type correlation structure and significantly outperform their baseline heuristics for ordinary networks with a single connection type.
\end{abstract}

\section{Introduction}

Networks are powerful representations of interactions in complex systems with a wide range of applications in biology, physics, sociology, engineering and computer science. Modeling interactions between entities as links between nodes in a graph allows us to leverage formal methods to understand influence, community structure and other patterns, make predictions about future interactions and detect unusual activity. The study of networks and their applications has thus become a major focus of many scientific disciplines in recent decades.

Since the advent of large-scale online social networks, the \emph{link prediction problem} \cite{libennowell-04} has received increased attention. Link prediction is usually defined in terms of the following two interrelated problems:
\begin{itemize}
    \item Given a current snapshot of a network at the present time, what new connections are likely to develop in the future?
    \item Given an incomplete network, what connections are likely to be actually present but missing from the graph?
\end{itemize}
\noindent Link prediction has numerous applications including social network recommendation systems for new friends or individuals to follow \cite{parvathy:17},
predicting protein and metabolic interactions in biological networks \cite{wang:13},
finding experts and predicting collaborations in scientific co-authorship networks \cite{libennowell-04},
identifying hidden interactions of criminal organizations \cite{berlusconi:16} and predicting future routes in transit systems \cite{lu:11}.

Most of the existing link prediction literature focuses on ordinary networks which represent a single type of interaction between entities. In many complex systems, however, we observe multiple types of interactions. For example, individuals may interact using multiple social networks and cities and transit stations may be linked via different carriers, lines or modes of transit. In order to apply standard techniques, these multiple interactions must either be conflated to a single type, which is not appropriate if they are sufficiently dissimilar, or the analysis must be restricted to only one type of interaction. This is limiting since conflation restricts our ability to predict the type of future or missing interactions while using only a single interaction type fails to leverage additional useful information gleaned from other types of interactions in the network.

\emph{Multiplex networks} are graphical structures that can represent multiple types of interactions between entities \cite{kivela:14}. In multiplex networks, connections between entities occur at a \emph{layer} of the network, which represents a specific interaction type. These networks can be visualized as either a single graph with multiple edge types or a set of ordinary (\emph{single-layer} or \emph{monoplex}) graphs with the same nodes but different edges, each corresponding to a different layer. Figure 1 depicts a multiplex network representing 3 types of interactions among 9 entities. In this example, $X$ and $V$ are connected in layer 1, which might correspond to a specific social network, but are not connected in the other layers, which might correspond to other social networks.\vspace{.4cm}
\begin{figure}
\begin{tikzpicture}[multilayer=3d]
\SetLayerDistance{-1}
\begin{Layer}[layer=1]
\Plane[x=-1.45,y=-1.4,width=2.9,height=2.9,NoBorder,opacity=.1,color=blue];
\node at (-.5,-1.45)[below right]{Layer 1};
\end{Layer}
\begin{Layer}[layer=2]
\Plane[x=-1.45,y=0.3,width=2.9,height=2.9,NoBorder,opacity=.1,color=green]
\node at (1.25,-1.45)[below right]{Layer 2};
\end{Layer}
\begin{Layer}[layer=3]
\Plane[x=-1.45,y=2.05,width=2.9,height=2.9,NoBorder,opacity=.1,color=orange]
\node at (3,-1.45)[below right]{Layer 3};
\end{Layer}
\Vertex[x=-.9,y=.9,layer=1,size=.5,style={shading=ball,opacity=.2},opacity=.2,fontscale=1.25,label=X,color=blue]{X1}
\Vertex[x=0,y=.9,layer=1,size=.5,style={shading=ball,opacity=.2},opacity=.2,fontscale=1.25,label=Y,color=blue]{Y1}
\Vertex[x=.9,y=.9,layer=1,size=.5,style={shading=ball,opacity=.2},opacity=.2,fontscale=1.25,label=Z,color=blue]{Z1}
\Vertex[x=-.9,y=0,layer=1,size=.5,style={shading=ball,opacity=.2},opacity=.2,fontscale=1.25,label=U,color=blue]{U1}
\Vertex[x=0,y=0,layer=1,size=.5,style={shading=ball,opacity=.2},opacity=.2,fontscale=1.25,label=V,color=blue]{V1}
\Vertex[x=.9,y=0,layer=1,size=.5,style={shading=ball,opacity=.2},opacity=.2,fontscale=1.25,label=W,color=blue]{W1}
\Vertex[x=-.9,y=-.9,layer=1,size=.5,style={shading=ball,opacity=.2},opacity=.2,fontscale=1.25,label=P,color=blue]{P1}
\Vertex[x=0,y=-.9,layer=1,size=.5,style={shading=ball,opacity=.2},opacity=.2,fontscale=1.25,label=Q,color=blue]{Q1}
\Vertex[x=.9,y=-.9,layer=1,size=.5,style={shading=ball,opacity=.2},opacity=.2,fontscale=1.25,label=R,color=blue]{R1}
\Vertex[x=.85,y=.9,layer=2,size=.5,style={shading=ball,opacity=.2},opacity=.2,fontscale=1.25,label=X,color=green]{X2}
\Vertex[x=1.75,y=.9,layer=2,size=.5,style={shading=ball,opacity=.2},opacity=.2,fontscale=1.25,label=Y,color=green]{Y2}
\Vertex[x=2.65,y=.9,layer=2,size=.5,style={shading=ball,opacity=.2},opacity=.2,fontscale=1.25,label=Z,color=green]{Z2}
\Vertex[x=.85,y=0,layer=2,size=.5,style={shading=ball,opacity=.2},opacity=.2,fontscale=1.25,label=U,color=green]{U2}
\Vertex[x=1.75,y=0,layer=2,size=.5,style={shading=ball,opacity=.2},opacity=.2,fontscale=1.25,label=V,color=green]{V2}
\Vertex[x=2.65,y=0,layer=2,size=.5,style={shading=ball,opacity=.2},opacity=.2,fontscale=1.25,label=W,color=green]{W2}
\Vertex[x=.85,y=-.9,layer=2,size=.5,style={shading=ball,opacity=.2},opacity=.2,fontscale=1.25,label=P,color=green]{P2}
\Vertex[x=1.75,y=-.9,layer=2,size=.5,style={shading=ball,opacity=.2},opacity=.2,fontscale=1.25,label=Q,color=green]{Q2}
\Vertex[x=2.65,y=-.9,layer=2,size=.5,style={shading=ball,opacity=.2},opacity=.2,fontscale=1.25,label=R,color=green]{R2}
\Vertex[x=2.6,y=.9,layer=3,size=.5,style={shading=ball,opacity=.2},opacity=.2,fontscale=1.25,label=X,color=orange]{X3}
\Vertex[x=3.5,y=.9,layer=3,size=.5,style={shading=ball,opacity=.2},opacity=.2,fontscale=1.25,label=Y,color=orange]{Y3}
\Vertex[x=4.4,y=.9,layer=3,size=.5,style={shading=ball,opacity=.2},opacity=.2,fontscale=1.25,label=Z,color=orange]{Z3}
\Vertex[x=2.6,y=0,layer=3,size=.5,style={shading=ball,opacity=.2},opacity=.2,fontscale=1.25,label=U,color=orange]{U3}
\Vertex[x=3.5,y=0,layer=3,size=.5,style={shading=ball,opacity=.2},opacity=.2,fontscale=1.25,label=V,color=orange]{V3}
\Vertex[x=4.4,y=0,layer=3,size=.5,style={shading=ball,opacity=.2},opacity=.2,fontscale=1.2,label=W,color=orange]{W3}5
\Vertex[x=2.6,y=-0.9,layer=3,size=.5,style={shading=ball,opacity=.2},opacity=.2,fontscale=1.25,label=P,color=orange]{P3}
\Vertex[x=3.5,y=-0.9,layer=3,size=.5,style={shading=ball,opacity=.2},opacity=.2,fontscale=1.25,label=Q,color=orange]{Q3}
\Vertex[x=4.4,y=-0.9,layer=3,size=.5,style={shading=ball,opacity=.2},opacity=.2,fontscale=1.25,label=R,color=orange]{R3}
\Edge[style={dashed},opacity=.3](X1)(X2)
\Edge[style={dashed},opacity=.3](X2)(X3)
\Edge[style={dashed},opacity=.3](Y1)(Y2)
\Edge[style={dashed},opacity=.3](Y2)(Y3)
\Edge[style={dashed},opacity=.3](Z1)(Z2)
\Edge[style={dashed},opacity=.3](Z2)(Z3)
\Edge[style={dashed},opacity=.3](U1)(U2)
\Edge[style={dashed},opacity=.3](U2)(U3)
\Edge[style={dashed},opacity=.3](V1)(V2)
\Edge[style={dashed},opacity=.3](V2)(V3)
\Edge[style={dashed},opacity=.3](W1)(W2)
\Edge[style={dashed},opacity=.3](W2)(W3)
\Edge[style={dashed},opacity=.3](P1)(P2)
\Edge[style={dashed},opacity=.3](P2)(P3)
\Edge[style={dashed},opacity=.3](Q1)(Q2)
\Edge[style={dashed},opacity=.3](Q2)(Q3)
\Edge[style={dashed},opacity=.3](R1)(R2)
\Edge[style={dashed},opacity=.3](R2)(R3)
\Edge[lw=1pt,color=gray](X1)(U1)
\Edge[lw=1pt,color=gray](X1)(Y1)
\Edge[lw=1pt,color=gray](X1)(V1)
\Edge[lw=1pt,color=gray](Y1)(Z1)
\Edge[lw=1pt,color=gray](V1)(W1)
\Edge[lw=1pt,color=gray](V1)(Z1)
\Edge[lw=1pt,color=gray](V1)(P1)
\Edge[lw=1pt,color=gray](P1)(Q1)
\Edge[lw=1pt,color=gray](Q1)(R1)
\Edge[lw=1pt,color=gray](V1)(Q1)
\Edge[lw=1pt,color=gray](X2)(U2)
\Edge[lw=1pt,color=gray](X2)(Y2)
\Edge[lw=1pt,color=gray](U2)(V2)
\Edge[lw=1pt,color=gray](Y2)(Z2)
\Edge[lw=1pt,color=gray](W2)(V2)
\Edge[lw=1pt,color=gray](U2)(P2)
\Edge[lw=1pt,color=gray](P2)(Q2)
\Edge[lw=1pt,color=gray](Q2)(R2)
\Edge[lw=1pt,color=gray](X3)(Y3)
\Edge[lw=1pt,color=gray](U3)(Y3)
\Edge[lw=1pt,color=gray](U3)(V3)
\Edge[lw=1pt,color=gray](V3)(W3)
\Edge[lw=1pt,color=gray](V3)(Z3)
\Edge[lw=1pt,color=gray](V3)(R3)
\Edge[lw=1pt,color=gray](V3)(Q3)
\Edge[lw=1pt,color=gray](P3)(Q3)
\end{tikzpicture}
\caption{Multiplex network with 3 layers and 9 nodes}\label{mplx}
\end{figure}
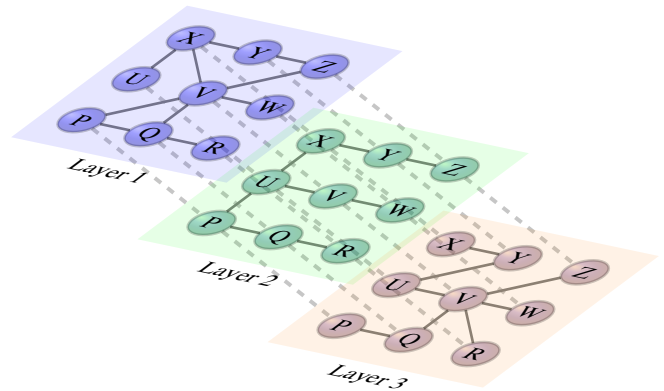

\vspace{-.4cm}While interest in multiplex networks has grown across communities and there is prior work investigating centrality and community structure \cite{kivela:14,kanawati:15}, there is limited existing work on link prediction. In contrast to the multitude of simple heuristics for link prediction in ordinary networks, which have been thoroughly investigated empirically \cite{libennowell-04} and theoretically \cite{sarkar:2011}, we are not aware of any general heuristics for link prediction at specific multiplex network layers.

We propose a novel general framework and three families of heuristics for link prediction in multiplex networks which take advantage of strong \emph{cross-layer correlation} structure, which has been observed in many real-world complex systems \cite{nicosia:15}. We show that the performance of the proposed heuristics increases with the strength of cross-layer correlations and they outperform their baseline heuristics in synthetically generated and real world multiplex networks. 

\section{Background} 

We represent an ordinary undirected graph as  $\mathcal{G}=\langle\mathcal{V},\mathcal{E}\rangle$, where $\mathcal{V}$ is a set of nodes and $\mathcal{E}$ a set of edges. Distinct nodes $v, v' \in  \mathcal{V}$ are \emph{neighbors} if they are connected by an edge in $\mathcal{E}$; otherwise, they are \emph{non-neighbors}. $\mathcal{N}(v)$ represents the set of neighbors of $v \in \mathcal{V}$.
The \emph{degree} of a node is the cardinality of its neighbors set. A \emph{path} between $u, w \in \mathcal{V}$ is an ordered set $\langle v_1, \dots, v_n\rangle \subset \mathcal{V}$ such that $u \in \mathcal{N}(v_1)$, $w \in \mathcal{N}(v_n)$ and for $1 \leq i < n$, $v_i \in \mathcal{N}(v_{i+1})$.
We restricted our analysis to \emph{undirected} graphs in this paper.

\subsection{Link Prediction in Single-Layer Networks}

\cite{libennowell-04} provides the first comprehensive introduction to and analysis of the link prediction problem. Recent surveys include \cite{lu:11} and \cite{martinez:16}.

Link prediction is often posed as a ranking problem where pairs of non-neighbors are scored according to the predicted likelihood of a future or missing connection and the top $k$ highest scoring pairs are selected. It can also be posed as a binary classification problem where the class of a pair of nodes is whether or not a link exists. 

The most extensively studied link prediction techniques are based on \emph{similarity heuristics}, which score pairs of nodes according to topological features of the network related to coherent assumptions about their similarity \cite{martinez:16}. Most similarity heuristics are adapted from techniques from graph theory and social network analysis \cite{libennowell-04}. We define and discuss some of the most common heuristics below. A more comprehensive list is provided in \cite{martinez:16}.

\emph{Neighbor-based} heuristics are based on the idea that a link is most likely to exist between nodes $v$ and $v'$ whose sets of neighbors significantly overlap. This property has been empirically observed in real world networks \cite{newman:01}. The heuristic which most directly implements this concept is \emph{Common Neighbors (CN)}, which is simply the cardinality of the intersection of neighbor sets \cite{newman:01}:
\[
CN\left(v,v\right)=\left|\mathcal{N}(v) \cap \mathcal{N}\left(v'\right)\right|
\]
\noindent A related measure is the \emph{Jaccard Coefficient (JC)}, which is the ratio of this intersection to the union of the neighbor sets:
\[
JC\left(v,v'\right)=\cfrac{\left|\mathcal{N}\left(v\right) \cap \mathcal{N}\left(v'\right)\right|}{\left|\mathcal{N}\left(v\right) \cup \mathcal{N}\left(v'\right)\right|}
\]
\noindent \emph{Resource Allocation (RA)} and \emph{Adamic-Adar (AA)} \cite{adamic:01} score links inversely proportional to the number of neighbors of each common neighbor of two nodes:
\begin{align*}
RA\left(v,v'\right) &= \sum\limits_{u \in \mathcal{N}\left(v\right) \cap \mathcal{N}\left(v'\right)} \cfrac{1}{\left|\mathcal{N}\left(u\right)\right|}\\
AA\left(v,v'\right)&=\sum\limits_{u \in \mathcal{N}\left(v\right) \cap \mathcal{N}\left(v'\right)} \cfrac{1}{\log \left|\mathcal{N}\left(u\right)\right|}
\end{align*}
\noindent \emph{Preferential Attachment (PA)}, adapted from the Barab\'asi-Albert network growth model \cite{barabasi:99}, is the product of node degrees \cite{barabasi:02}:
\[
PA\left(v,v'\right)=\left|\mathcal{N}\left(v\right)\right| \times \left|\mathcal{N}\left(v'\right)\right|
\]
\noindent The \emph{Product of Clustering Coefficient (PCC)} scores the likelihood of a link proportional to the product of the nodes' \emph{clustering coefficients}, or number of links between nodes that are neighbors proportional to the total possible links between those nodes:
\[
PCC\left(v,v'\right)=\prod\limits_{w \in \{v,v'\}}\cfrac{2\left|\left\{u,u' \in \mathcal{N}(w): u' \in \mathcal{N}(u)\right\}\right|}{\left|\mathcal{N}(w)\right|\left(\left|\mathcal{N}(w)\right|-1\right)}
\]
These heuristics are simple, interpretable, computationally efficient and highly parallelizable. Their primary disadvantage is they do not consider paths between nodes without common neighbors \cite{martinez:16}.

\emph{Path-based} heuristics consider all paths between nodes. The \emph{Katz Score (KS)} sums over all paths between two nodes and applies exponential dampening according to path lengths for specified $\beta$ \cite{katz:53}:
\[
KS\left(v,v'\right)=\sum\limits_{\mathbf{p}\in paths(v,v')} \beta^{\left|\mathbf{p}\right|}
\]
\noindent Smaller $\beta$ values result in a heuristic similar to neighbor-based approaches. \emph{Rooted PageRank (RPR)}, based on the PageRank measure for website authoritativeness \cite{brin:98}, is defined as the stationary probability that a random walk from $v$ to $v'$ with probability $1-\alpha$ of returning to $v$ and otherwise moving to a random neighbor reaches $v'$, represented as $[\pi_v]_{v'}$ \cite{tong:06}:
\[
RPR\left(v,v'\right)=[\pi_v]_{v'} + [\pi_{v'}]_v
\]

While comprehensive studies of link prediction have focused on unsupervised prediction using these heuristics, supervised and optimization-based approaches have also been considered. Most of these use similarity heuristics as features, sometimes with additional information, to train a classifier \cite{hasan:06,cukierski:11} or learn a weighting function \cite{bliss:14}. Empirical studies have found simple neighbor-based heuristics often perform as well or better than more complex methods \cite{martinez:16,libennowell-04}. There are some theoretical justifications for their success \cite{sarkar:2011}.

\subsection{Multiplex Networks}

For decades, different disciplines have proposed systems which organize different types of connections between entities, but only recently have there been significant efforts to develop general frameworks for studying networks with multiple layers or types of connections \cite{kivela:14}. This increased interest has resulted in disparate terminology and formulations of multiplex networks and related network representations.

One popular formulation of multiplex networks is a graph with multiple edge types which each correspond to different layers. We can represent a multiplex network as $\mathcal{G}=\langle \mathcal{V}, \mathcal{E}, \mathcal{T}\rangle$ where $\mathcal{T}$ is a set of edge types and each edge in $\mathcal{E}$ is between $v, v' \in \mathcal{V}$ and of type $t \in \mathcal{T}$. Other formulations allow for different node sets and edges which cross layers \cite{brodka-17}, sometimes referred to as \emph{heterogeneous networks}.
In our setup, edges are always within the same layer and node sets are common across layers. We can thus equivalently represent a multiplex network as a set of graphs with the same node set, where each graph represents a different layer in the network, e.g. $\bm{\mathcal{G}} = \langle \mathcal{G}^1,\dots,\mathcal{G}^k\rangle$.

\section{Multiplex Network Link Prediction Heuristics}

The framework we propose for specifying heuristics for link prediction in multiplex networks is inspired by the rich connection type correlation structures that have been empirically observed in many real world complex systems \cite{nicosia:15}. We provide a general approach to defining heuristics in terms of topological features across layers of a multiplex network weighted according to this structure. The motivation for this approach is that real world multiplex networks often contain sets of layers which are highly (positively or negatively) correlated but many pairs of layers which are not strongly correlated. When predicting links at a given layer, we would like to take advantage of structural information from other layers which are highly correlated, but ignore layers where correlations are weak.

\subsection{Cross-Layer Correlation}

First, we define correlation between multiplex network layers. Previous work comparing layers primarily considers layer similarity in terms of shared edges and hubs (high degree nodes) \cite{brodka-17,nicosia:15}; however, for specific problems, it may be appropriate to consider higher-order structural features \cite{benson:19}, e.g. shared triangles, or other contextual information. Our framework is general enough that it can be adapted to the specific needs of a particular application, allowing the specification of both relevant features and metrics used to define correlation.

As an initial step, we define a \emph{property matrix}, following \cite{brodka-17}, for a multiplex network which specifies the relevant features to consider cross-layer correlation in terms of. For example, to calculate cross-layer correlation in terms of shared edges, we construct the following property matrix $\mathbf{P}$ for the multiplex network depicted in Figure 1:
\[
\begin{blockarray}{rccccc}
&& X-Y & X-U & X-V & \\
\begin{block}{r[c@{}ccc@{}c}
\text{Layer 1} && 1 & 1 & 1 &\ \dots \\
\mathbf{P}=\quad
\text{Layer 2}&& 1 & 1 & 0 &\ \dots \\
\text{Layer 3} && 1 & 0 & 0 &\ \dots \\
\end{block}
\end{blockarray}
\]
\noindent Rows in $\mathbf{P}$ represent layers, and columns represent unique node pairs.
Entries of 1 or 0 indicate the presence or lack of an edge, respectively.
Similarly, to compare layers in terms of shared hubs,
we make the columns represent nodes and have the entries indicate the node degree in each layer. For a property matrix $\mathbf{P}$ we use $\mathbf{p}^i$ to indicate the \emph{property vector} for the $i$th layer and $p^i_j$ the value in the $j$th column for layer $i$. By convention, all vectors are treated as column vectors. When property matrices/vectors are defined in terms of shared edges or shared hubs, we refer to them as \emph{edge property matrices/vectors} or \emph{degree property matrices/vectors}, respectively.

We next construct a \emph{cross-layer correlation matrix} $\mathbf{C}$ from a $k \times x$ property matrix $\mathbf{P}$ by setting the diagonal entries in $\mathbf{C}$ to 1 and the off-diagonal entries $c_{i,j}$ to the value resulting from some correlation metric applied to the property vectors $\mathbf{p}^i$ and $\mathbf{p}^j$. For example, using Pearson correlation we get the following for the off-diagonals, where we represent the mean taken with respect to a property vector $i$ as $\bar{p}^i = \frac{1}{x}\sum_{j=1}^{x} p_j^i$:
\[
c_{i,j}=\frac{\left(\mathbf{p}^i-\bar{p}^i\right)'\left(\mathbf{p}^j-\bar{p}^j\right)}{\sqrt{\left(\mathbf{p}^i-\bar{p}^i\right)'\left(\mathbf{p}^i-\bar{p}^i\right)\left(\mathbf{p}^i-\bar{p}^j\right)'\left(\mathbf{p}^j-\bar{p}^j\right)}}
\]
While Pearson correlation is an appropriate metric for edge property matrices, Spearman (rank-based) correlation is more appropriate for degree property matrices since denser layers may have the same rank ordering of hubs, but with different degrees. We focus on \emph{correlation} metrics as opposed to general distance metrics since they distinguish positive from negative correlation, which has been observed in real world networks and which we account for in our proposed heuristics. 

\subsection{Multiplex Network Heuristics}

We now propose three multiplex network heuristics which use cross-layer correlation structure to weight features observed across layers. Each are defined in terms of a specified cross-layer correlation matrix $\mathbf{C}$, allowing for the use of any property matrix and correlation metric. First, we define the following normalization for a layer $i$ and $\mathbf{C}$:
\[
Z^i_\mathbf{C} = \sum_{l=1}^k \left|c_{i,l}\right|.
\]

The first and simplest heuristic, \emph{Count and Weight by Correlation (CWC)}, counts the number of layers which contain a link between two nodes and weights that count according to the cross-layer correlations.
\begin{heuristic}[Count and Weight by Correlation]
\label{def:cwc}
Let $\bm{\mathcal{G}}=\langle \mathcal{G}_1,\dots,\mathcal{G}_k\rangle$ be a multiplex network with edge property vectors $\mathbf{e}^1,\dots,\mathbf{e}^k$ and cross-layer correlation matrix $\mathbf{C}$. CWC is defined for a layer $i$ and a possible edge represented by an edge property vector index $j$ as follows:
\begin{align*}
\cfrac{1}{Z^i_\mathbf{C}}{\ } \sum\limits_{l=1}^k 
\begin{cases} 
      e^i_j c_{i,l},  &  c_{i,l} > 0 \\
      \left(1-e^i_j\right)\left|c_{i,l}\right|,  & c_{i,l} < 0
\end{cases}
\end{align*}
\end{heuristic}
\noindent For example, to consider a link in the multiplex network in Figure 1 between $X$ and $V$ at layer 2 using CWC, we would proceed with the following calculation (assuming only positive correlations):
\[
\frac{1}{Z^2_\mathbf{C}}\left(1 \times c_{2,1} + 0 \times c_{2,3}\right)
\]
Only $c_{2,1}$ receives weight in the numerator since $X$ and $V$ are connected in layer 1 but not in layer 3.

CWC encodes the intuition that correlated layers should have similar links: the more correlated a layer which does not contain a particular link is to another layer which does contain that link, the more likely it is that link is missing or will develop in the future. CWC also takes anti-correlation into account: a link is more likely to be predicted if it is missing from a layer which is anti-correlated. Despite its simplicity, this heuristic performs extremely well in practice. 

The second heuristic, \emph{Correlation Weighted Heuristic (CWH)}, extends the heuristics discussed in the previous section to the multiplex domain by applying them across layers of a multiplex network and weighting them according to cross-layer correlations.
While empirical studies have found that no particular monoplex heuristic consistently outperforms all others \cite{libennowell-04}, there may be problem-specific reasons to prefer a particular heuristic.
For example, if we know there are few long paths between nodes, a neighbor-based heuristic is likely to perform at least as well as a path-based heuristic at a lower computational cost.
Taking this into consideration, CWH allows any monoplex heuristic to be extended to multiplex networks.
\begin{heuristic}[Correlation Weighted Heuristic]
\label{def:cwh}
Let $\bm{\mathcal{G}}=\langle \mathcal{G}_1,\dots,\mathcal{G}_k\rangle$ be a multiplex network with cross-layer correlation matrix $\mathbf{C}$. Let $h^l_j$ be a heuristic for monoplex networks evaluated at layer $l$ of $\bm{\mathcal{G}}$ for a possible edge represented by an edge property vector index $j$. Then, CWH is defined for a layer $i$ and possible edge index $j$ as follows:
\begin{align*}
\cfrac{1}{Z^i_\mathbf{C}}{\ } \sum\limits_{l=1}^k 
\begin{cases} 
      h^l_j c_{i,l},  &  c_{i,l} > 0 \\
      \left(1-h^l_j\right)\left|c_{i,l}\right|,  & c_{i,l} < 0
 \end{cases}
\end{align*}
\end{heuristic}
\noindent For example, to consider a link in the multiplex network in Figure 1 between $X$ and $V$ at layer 2 using CWH with Common Neighbors as the monoplex heuristic, we would proceed with the following calculation (assuming only positive correlations):
\[
\frac{1}{Z^2_\mathbf{C}}\left[CN^1\left( X, V \right) \times c_{2,1} + CN^2\left( X, V \right) + CN^3\left( X, V \right) \times c_{2,3}\right]
\]
CWH is similarly based on the intuition that since existing monoplex heuristics have been shown to be predictive of missing and future links in single-layer networks, they should also be predictive in correlated layers of multiplex networks and this predictive power should increase based on the magnitude of correlations. Like CWC, CWH takes anti-correlation into account: links are more likely to be predicted if they are not strongly predicted by a monoplex heuristic in an anti-correlated layer. In our definition, we assume the monoplex heuristic $h$ is normalized to be within 0 and 1. 

The third heuristic combines the previous two ideas. For a given a monoplex heuristic, \emph{Count Correlation-Weighted Heuristics (CCWH)} counts the number of layers which contain a link between two nodes and weights that count according to both cross-layer correlations and the values resulting from evaluating the monoplex heuristic at each layer in the network.

\begin{heuristic}[Count Correlation-Weighted Heuristics]
Let $\bm{\mathcal{G}}=\langle \mathcal{G}_1,\dots,\mathcal{G}_k\rangle$ be a multiplex network with edge property vectors $\mathbf{e}^1,\dots,\mathbf{e}^k$ and cross-layer correlation matrix $\mathbf{C}$. Let $h^l_j$ be a similarity heuristic for monoplex networks evaluated at layer $l$ of $\bm{\mathcal{G}}$ for a possible edge represented by an edge property vector index $j$. Then, CCWH is defined for a layer $i$ and possible edge index $j$ as follows:
\begin{align*}
\cfrac{1}{Z^i_\mathbf{C}}{\ } \sum\limits_{l=1}^k 
\begin{cases} 
      h^{i}_j ,  &  i=l \\
      e^i_j h^{i}_j c_{i,l},  &  c_{i,l} > 0 \\
      \left(1-e^i_j\right)\left(1-h^{i}_j\right)\left|c_{i,l}\right|,  & c_{i,l} < 0
 \end{cases}
\end{align*}
\end{heuristic}
\noindent CCWH also accounts for negative correlation: links are more likely if they are not present in an anti-correlated layer and the magnitudes of these predictions are inversely proportional to the values of the heuristic evaluated at that layer. We also include the heuristic evaluated at the layer being predicted so that CCWH yields informative values even when there are no layers containing the edge being predicted.

\subsection{Expected Overlap Threshold for Layers}

One potential issue with using cross-layer correlation as weights in the proposed heuristics is the sensitivity of many correlation metrics to sample size error. When layers are not related, we may still observe small correlation values which add noise. This may be particularly acute when networks have small numbers of nodes but many layers. To improve empirical performance in such cases, we propose a thresholding method to ignore layers likely to only add noise.

One possibility is to simply ignore small values of correlation, but there is no clear guideline for setting the threshold for values to ignore. Instead, we propose a threshold for excluding layers based on properties of the two graphs being compared. If two graphs are related, especially in the context of link prediction, we expect them to have edges in common. Thus, we should expect that a layer $l$ used for predicting a link at another layer $i$ has at least as many overlapping edges with $i$ as a random graph. However, graphs with many edges are more likely to have overlapping edges so we should only consider random graphs with the same number of edges as the layer for which we are predicting links. The \ER{}  $\mathcal{G}_{n,m}$ random graph model \cite{erdos:59}, which uniformly considers all undirected graphs with $n$ nodes and $m$ edges. provides a theoretical framework for this comparison. Let $\mathcal{G}^i$ be an observed layer with $n$ nodes and $m^i$ edges at which we would like to predict links and let $\mathcal{G}^l$ be some other layer with $m^l$ edges. We define the expected number of overlapping edges (OE) in terms of the cosine distance between the edge property vector $\mathbf{p}^i$ for $\mathcal{G}^i$ and the edge property vector $\mathbf{p}^j$ for a random graph with $m^j=m^l$ edges generated according to the \ER{} random process:
\begin{align*}
\mathbb{E}\left( OE(\mathcal{G}^i, m^j)\right)=\mathbb{E}\left( \frac{\mathbf{p}^{i}{'}\mathbf{p}^j}{\sqrt{\mathbf{p}^{i}{'}\mathbf{p}^i\mathbf{p}^{j}{'}\mathbf{p}^j}} \middle\vert\ \mathcal{G}^i, m^j \right)
\end{align*}
To evaluate this quantity, we need the following lemma.
\newtheorem{lemma}{Lemma}
\begin{lemma}
Let $\mathcal{G} = \langle \mathcal{V}, \mathcal{E}\rangle$ be a graph with $n$ nodes, $m$ edges and edge property vector $\mathbf{p}$ that is generated according to an \ER{} $\mathcal{G}_{n,m}$ random graph process. Then, for $1 \leq i \leq \frac{n(n-1)}{2}$,
\[
\mathbb{E}\left( p_i\ |\ m \right) = \frac{2m}{n(n-1)}
\]
\end{lemma}
\begin{proof}
During the $k$th step of an \ER{} random process, the probability that a non-neighbor tuple $v, v' \in \mathcal{V}$ is not selected is
\[\frac{\frac{n(n-1)}{2}-k}{\frac{n(n-1)}{1}-k+1}\]
Therefore,
\begin{align*}
\mathbb{E}\left( p_i\ |\ m \right)& = (0)\mathbb{P}\left( p_i = 0\ |\ m \right) + (1)\mathbb{P}\left( p_i = 1\ |\ m \right)\\
 & = \mathbb{P}\left( p_i = 1\ |\ m \right)\\
 & = 1 - \mathbb{P}\left( p_i = 0\ |\ m \right)\\
 & = 1 - \prod_{k=1}^{m}\frac{\frac{n(n-1)}{2}-k}{\frac{n(n-1)}{2}-k+1}\\
 & = 1 - \frac{\frac{n(n-1)}{2}-m}{\frac{n(n-1)}{2}}\\
 & = \frac{2m}{n(n-1)}
\end{align*}
\end{proof}
\begin{theoremdef}
Let $\mathcal{G}^i = \langle \mathcal{V}, \mathcal{E}^i \rangle$ be an observed graph with $n$ nodes, $m^i$ edges and edge property vector $\mathbf{p}^i$ and $\mathcal{G}^j = \langle \mathcal{V}, \mathcal{E}^j \rangle$ a graph generated from an \ER{} $\mathcal{G}_{n,m^j}$ random process with edge property vector $\mathbf{p}^j$. Then,
\begin{align*}
\mathbb{E}\left( OE(\mathcal{G}^i, m^j)\right) = \frac{2\sqrt{m^im^j}}{n(n-1)}
\end{align*}
\end{theoremdef}
\begin{proof}
\begin{align*}
\mathbb{E}\left( OE(\mathcal{G}^i, m^j)\right)
& =\mathbb{E}\left( \frac{\mathbf{p}^{i}{'}\mathbf{p}^j}{\sqrt{\mathbf{p}^{i}{'}\mathbf{p}^i\mathbf{p}^{j}{'}\mathbf{p}^j}}\ \middle\vert\ \mathcal{G}^i, m^j \right)\\
& = \mathbb{E}\left( \frac{\mathbf{p}^{i}{'}\mathbf{p}^j}{\sqrt{m^im^j}}\ \middle\vert\ \mathcal{G}^i, m^j \right)\\
 & = \frac{1}{\sqrt{m^im^j}}\mathbb{E}\left( \sum_{k=1}^{\frac{n(n-1)}{2}}p^i_kp^j_k\ \middle\vert\ \mathcal{G}^i, m^j \right)\\
 & = \frac{1}{\sqrt{m^im^j}}\sum_{k=1}^{\frac{n(n-1)}{2}}p^i_k\mathbb{E}\left(p^j_k\ \middle\vert\ m^j \right)\\
 & = \frac{1}{\sqrt{m^im^j}}\sum_{k=1}^{\frac{n(n-1)}{2}}p^i_k\frac{2m^j}{n(n-1)}\\
 & = \frac{1}{\sqrt{m^im^j}}(m^i)\frac{2m^j}{n(n-1)}\\
 & = \frac{2\sqrt{m^im^j}}{n(n-1)}
\end{align*}
\end{proof}

The expected overlapping edges can be calculated whenever another layer is considered when evaluating a heuristics at a layer $i$ and ignored if the observed cosine distance between that layer's edge property vector and the edge property vector for layer $i$ is less than this quantity. However, we might also wish to consider only layers that are several standard deviations from a random graph. We thus need the following lemma to evaluate the second moment.
\begin{lemma}
Let $\mathcal{G} = \langle \mathcal{V}, \mathcal{E}\rangle$ be a graph with $n$ nodes, $m$ edges and edge property vector $\mathbf{p}$ generated according to an \ER{} $\mathcal{G}_{n,m}$ random graph process. Then, for $1 \leq i, j \leq \frac{n(n-1)}{2}$ such that $i\neq j$,
\[\mathbb{E}\left(p_ip_j\ |\ m \right) = \frac{4m(m-1)}{n(n-2)(n^2-1)}\]
\end{lemma}
\begin{proof}
First note that
\begin{align*}
m^2 & = \mathbb{E}\left(m^2 \right) =  \mathbb{E}\left(\left[\sum_{i=1}^{\frac{n(n-1)}{2}}p_i\right]\left[\sum_{j=1}^{\frac{n(n-1)}{2}}p_j\right]\middle\vert m \right) \\
& = \sum_{i=1}^{\frac{n(n-1)}{2}}\sum_{j=1}^{\frac{n(n-1)}{2}}\mathbb{E}\left(p_ip_j\middle\vert m \right)\\
& = \sum_{i=1}^{\frac{n(n-1)}{2}}\mathbb{E}\left(\left(p_i\right)^2\middle\vert m \right)+ \sum_{j,i=1 s.t. j \neq i}^{\frac{n(n-1)}{2}}\mathbb{E}\left(p_ip_j\middle\vert m \right)\\
& = \left(\frac{n(n-1)}{2}\right)\left(\frac{2m}{n(n-1)}\right)\\
 &\ \ \ \ \ \ + \left(\frac{n(n-1)}{2}\right)\left(\frac{n(n-1)}{2}-1\right)\mathbb{E}\left(p_ip_j\middle\vert m \right)\\
& = m + \frac{n^2(n-1)^2 -2n(n-1)}{4}\mathbb{E}\left(p_ip_j\middle\vert m \right)
\end{align*}
Factoring yields \[\mathbb{E}\left(p_ip_j\ |\ m \right) = \frac{4m(m-1)}{n(n-2)(n^2-1)}\]
\end{proof}
\begin{theoremdef}
Let $\mathcal{G}^i = \langle \mathcal{V}, \mathcal{E}^i \rangle$ be an observed graph with $n$ nodes, $m^i$ edges and edge property vector $\mathbf{p}^i$ and $\mathcal{G}^j = \langle \mathcal{V}, \mathcal{E}^j \rangle$ a graph generated from an \ER{} $\mathcal{G}_{n,m^j}$ random process with edge property vector $\mathbf{p}^j$. Then,
\begin{align*}
\mathbb{E}\left( \left[ OE(\mathcal{G}^i, m^j)\right]^2\right) = \frac{2}{n(n-1)} + \frac{4\left(m^i-1\right)\left(m^j-1\right)}{n(n-2)(n^2-1)}
\end{align*}
\end{theoremdef}
\begin{proof}\ Partition the indices $1, \dots, \frac{n(n-1)}{2}$ into $\langle \mathbf{I}^i_{+}, \mathbf{I}^i_{-} \rangle$ such that for $1 \leq k \leq \frac{n(n-1)}{2}$, $k \in \mathbf{I}^i_{+}$ if and only if $p^i_k = 1$ and $k \in\mathbf{I}^i_{-}$ if and on if  $p^i_k = 0$. Then,\\
$\mathbb{E}\left( \left[ OE(\mathcal{G}^i, m^j)\right]^2\right)$ 
\begin{align*}
& = \mathbb{E}\left(\left[\frac{\mathbf{p}^{i}{'}\mathbf{p}^j}{\sqrt{\mathbf{p}^{i}{'}\mathbf{p}^i\mathbf{p}^{j}{'}\mathbf{p}^j}}\right]^2\ \middle\vert\ \mathcal{G}^i, m^j \right)\\
& = \mathbb{E}\left(\left[ \frac{\mathbf{p}^{i}{'}\mathbf{p}^j}{\sqrt{m^im^j}}\right]^2\ \middle\vert\ \mathcal{G}^i, m^j \right)\\
& = \frac{1}{m^im^j}\mathbb{E}\left(\left[ \sum_{k=1}^{\frac{n(n-1)}{2}}p^i_kp^j_k\right]^2 \ \middle\vert\ \mathcal{G}^i, m^j \right)\\
& = \frac{1}{m^im^j}\mathbb{E}\left( \sum_{k=1}^{\frac{n(n-1)}{2}}\sum_{l=1}^{\frac{n(n-1)}{2}}p^i_kp^j_kp^i_lp^j_l \ \middle\vert\ \mathcal{G}^i, m^j \right)\\
& = \frac{1}{m^im^j}\mathbb{E}\left( \sum_{k \in \mathbf{I}^i_{+}}\sum_{l\in\mathbf{I}^i_{+}}p^j_kp^j_l \ \middle\vert\ m^j \right)\\
& = \frac{1}{m^im^j}\sum_{k \in \mathbf{I}^i_{+}}\sum_{l\in\mathbf{I}^i_{+}}\mathbb{E}\left(p^j_kp^j_l \ \middle\vert\ m^j \right)\\
& = \frac{1}{m^im^j}\sum_{k \in \mathbf{I}^i_{+}}\mathbb{E}\left(\left(p^j_k\right)^2 \ \middle\vert\ m^j \right) \\ 
&\ \ \ \ \ + \frac{1}{m^im^j}\sum_{k, l \in\mathbf{I}^i_{+} s.t. k\neq l}\mathbb{E}\left(p^j_kp^j_l \ \middle\vert\ m^j \right)\\
& = \frac{1}{m^im^j}\sum_{k \in \mathbf{I}^i_{+}}\frac{2m^j}{n(n-1)} \\ 
&\ \ \ \ \ + \frac{1}{m^im^j}\sum_{k,l \in\mathbf{I}^i_{+} s.t. k\neq l}\frac{4m^j\left(m^j-1\right)}{n\left(n-2\right)\left(n^2-1\right)}\\
& = \frac{1}{m^im^j}\left(m^i\right)\frac{2m^j}{n(n-1)}\\
&\ \ \ \ \ + \frac{1}{m^im^j}\left(m^i\left(m^i-1\right)\right)\frac{4m^j\left(m^j-1\right)}{n\left(n-2\right)\left(n^2-1\right)}\\
& = \frac{2}{n(n-1)} + \frac{4\left(m^i-1\right)\left(m^j-1\right)}{n(n-2)(n^2-1)}
\end{align*}
\end{proof}
\noindent The variance then follows as:
\[
\frac{2n(n-1)-4m^im^j}{n^2(n-1)^2} + \frac{4\left(m^i-1\right)\left(m^j-1\right)}{n(n-2)(n^2-1)}
\]

\section{Experiments}
\begin{figure*}
    \centering
    \includegraphics[scale=.352]{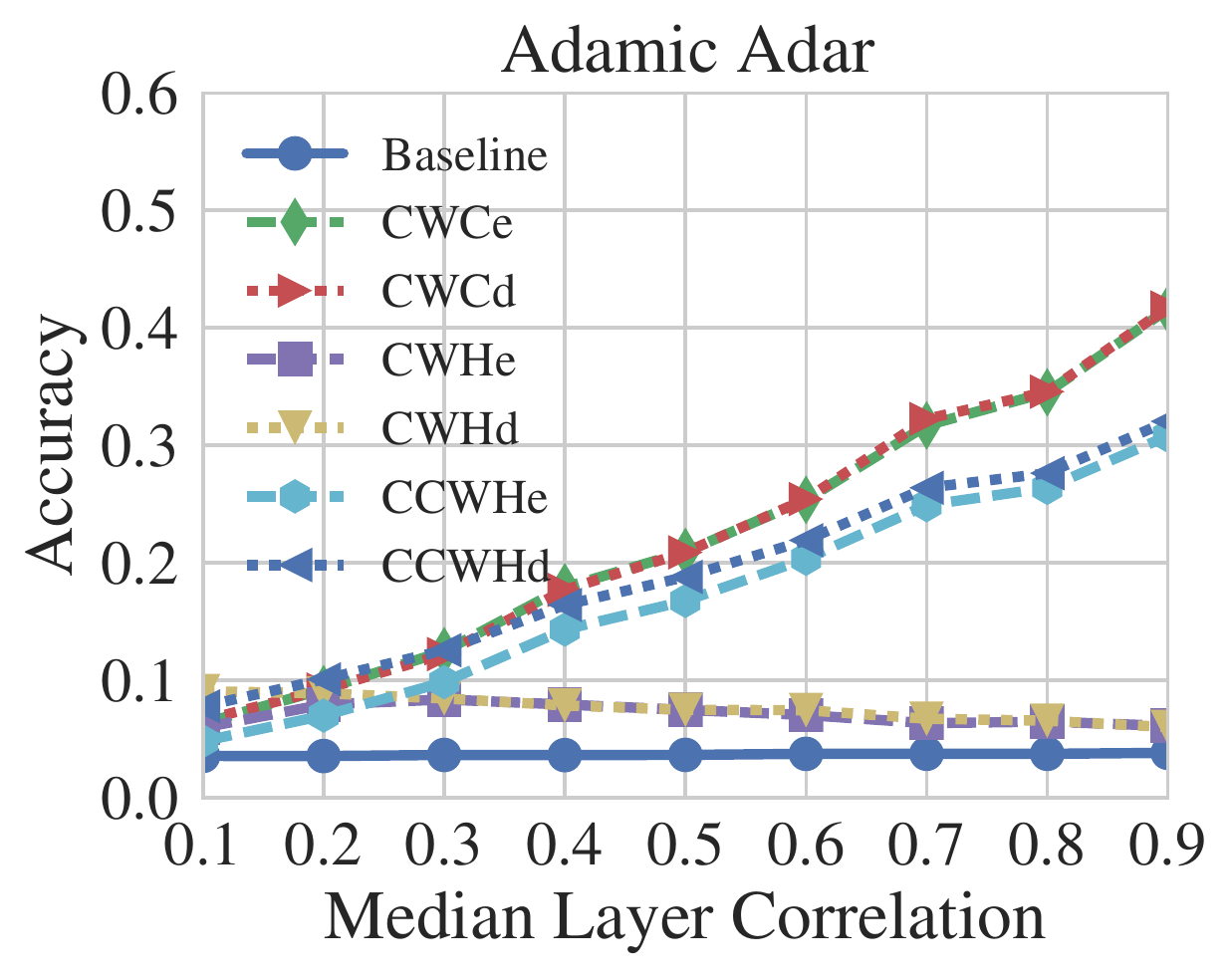}
    \includegraphics[scale=.352]{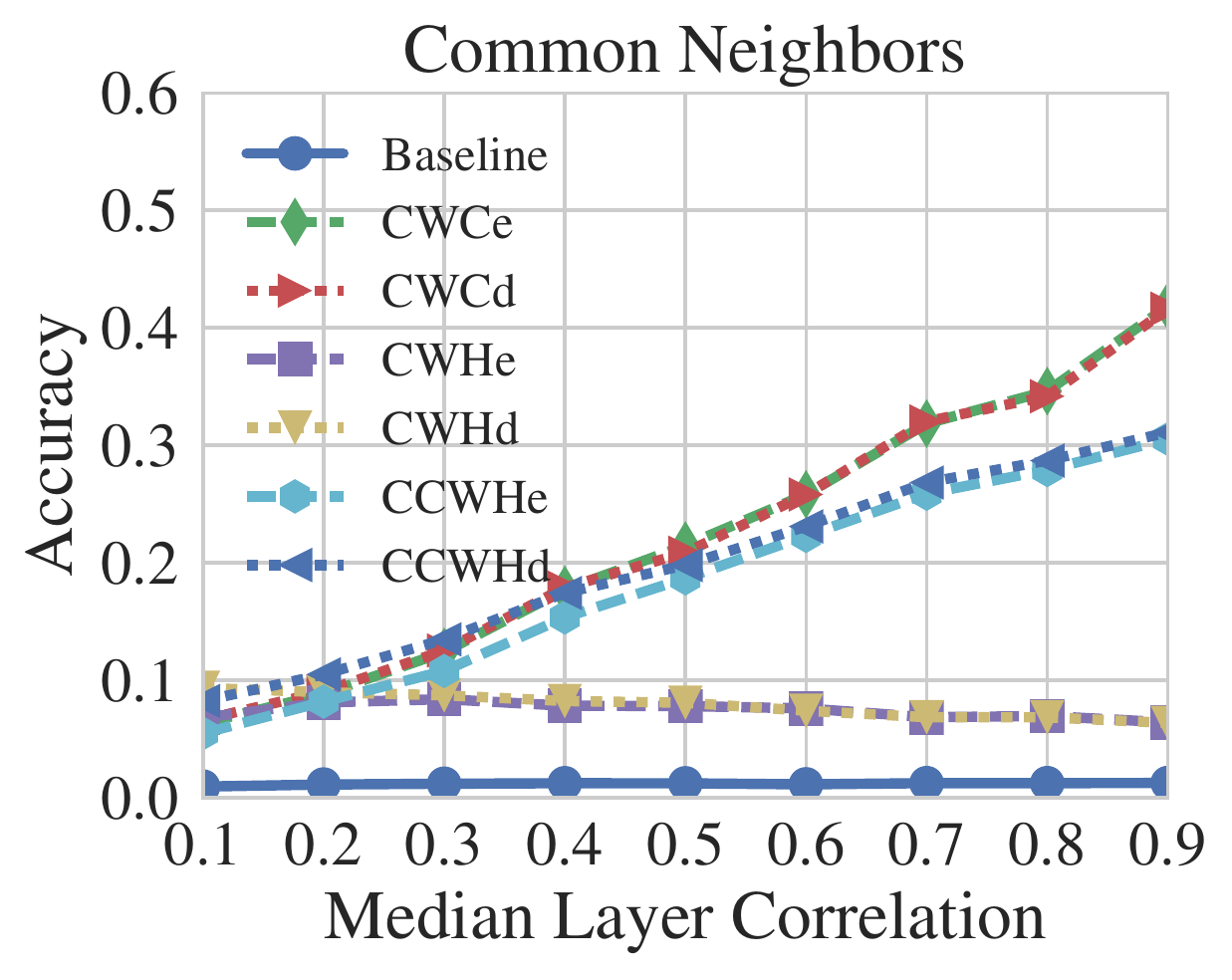}
    \includegraphics[scale=.352]{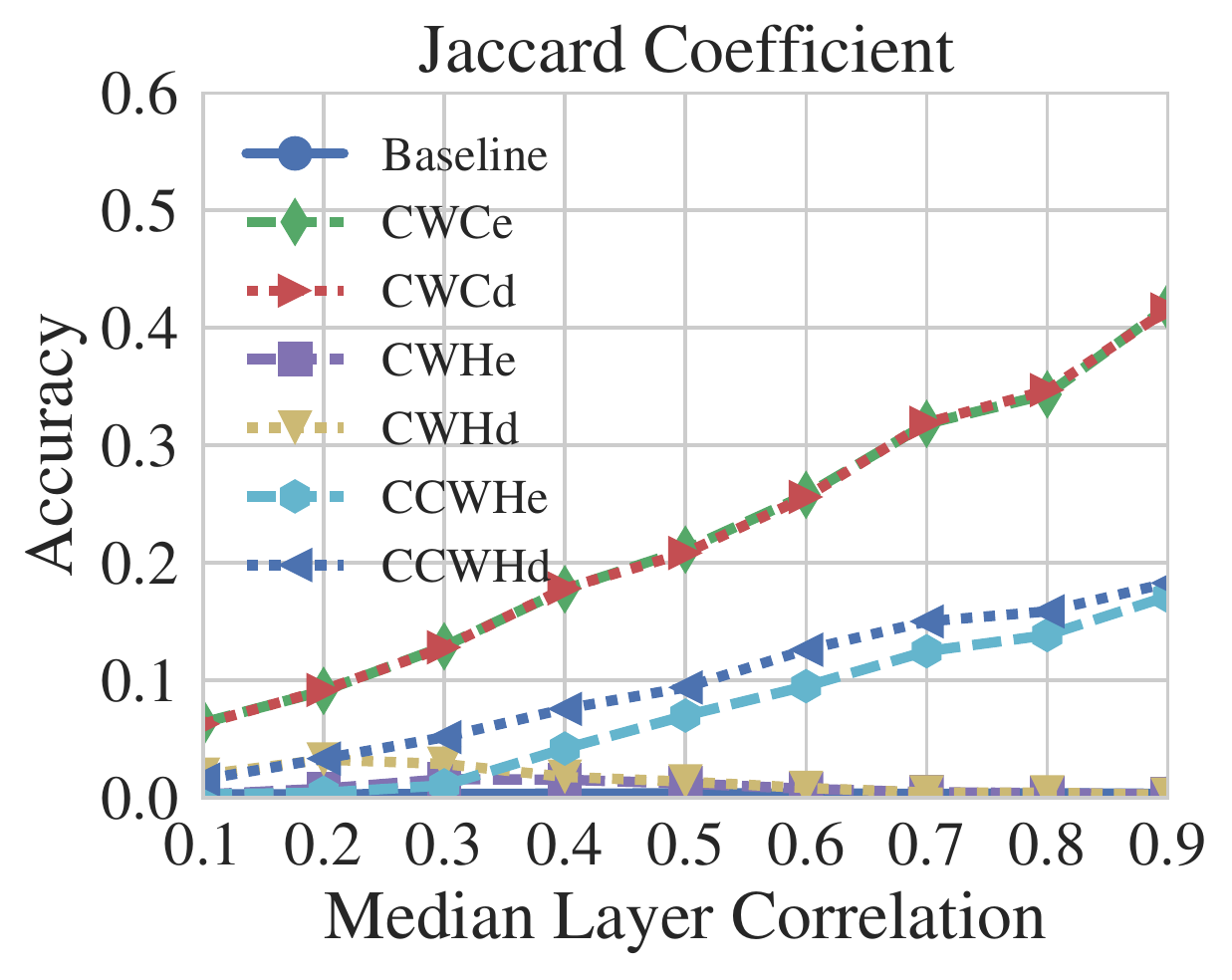}
    \includegraphics[scale=.352]{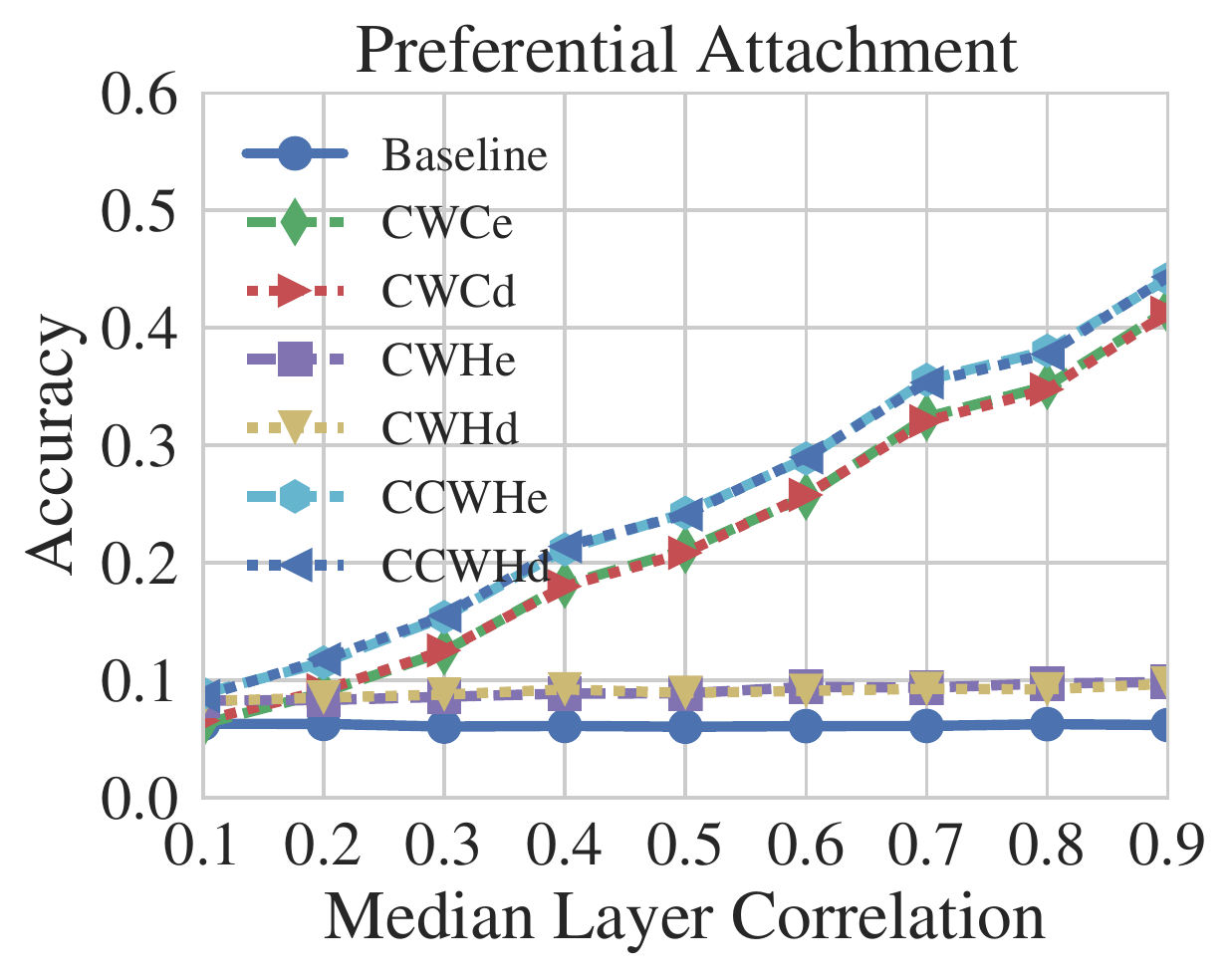}\\\ \\
    \includegraphics[scale=.352]{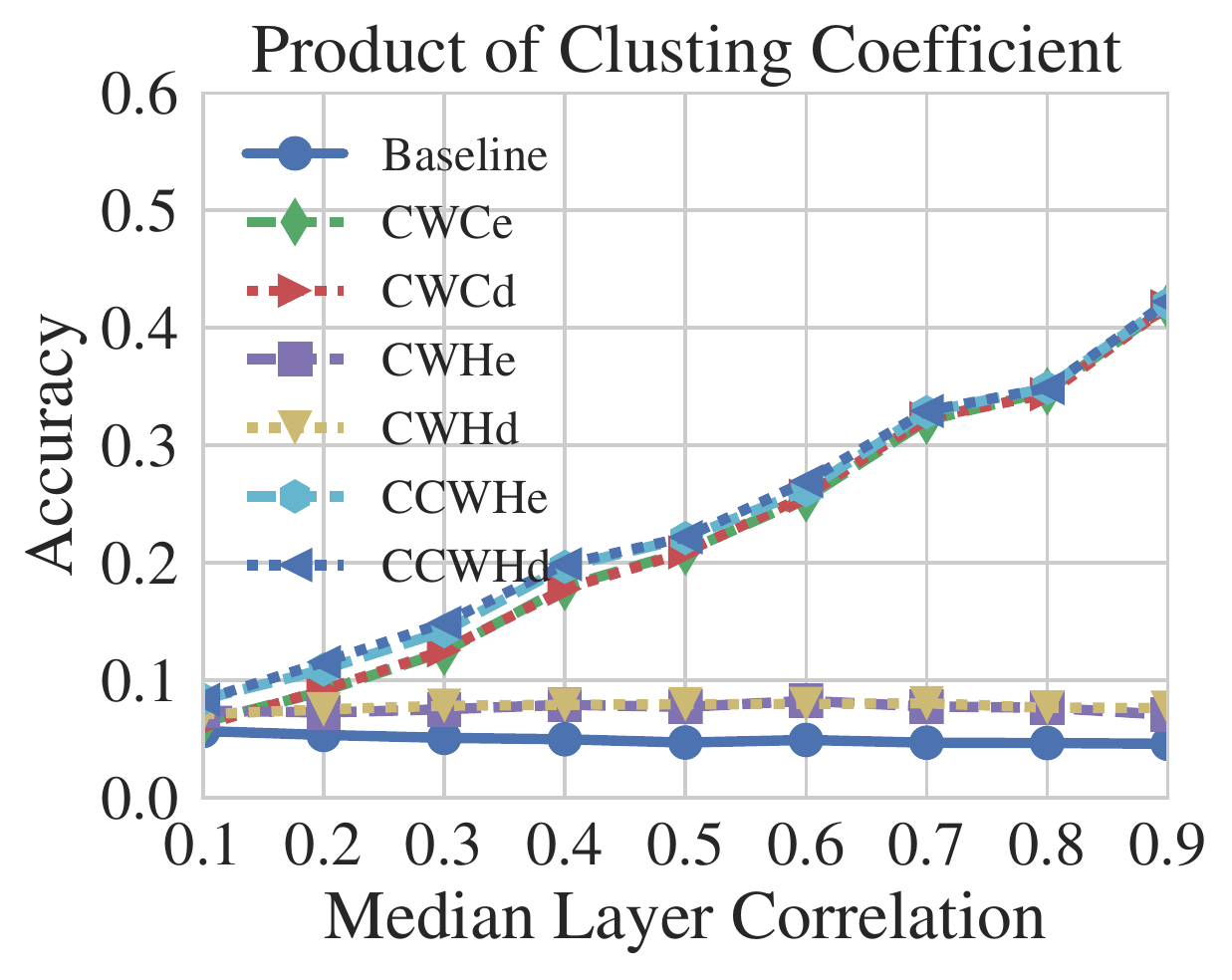}
    \includegraphics[scale=.352]{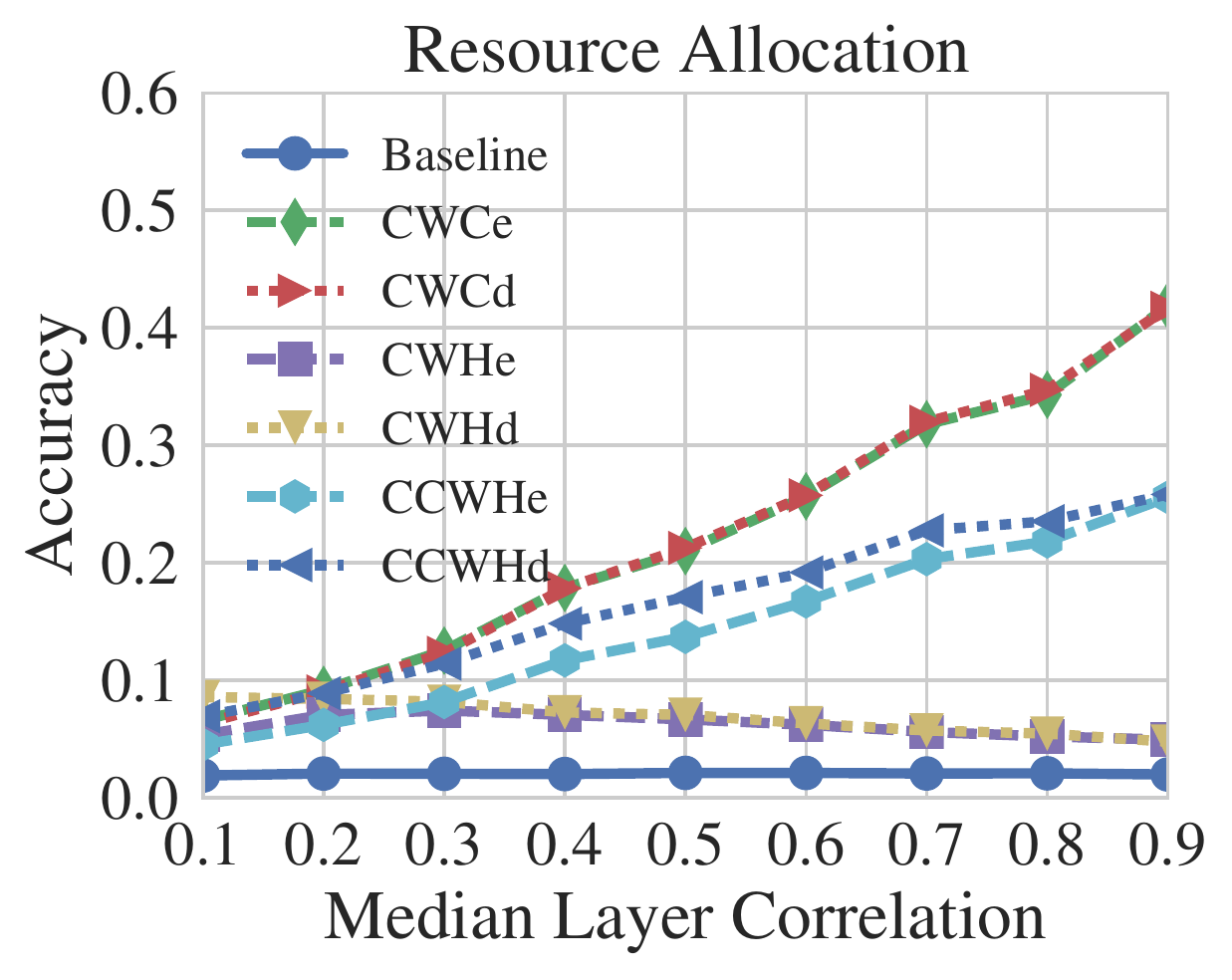}
    \includegraphics[scale=.352]{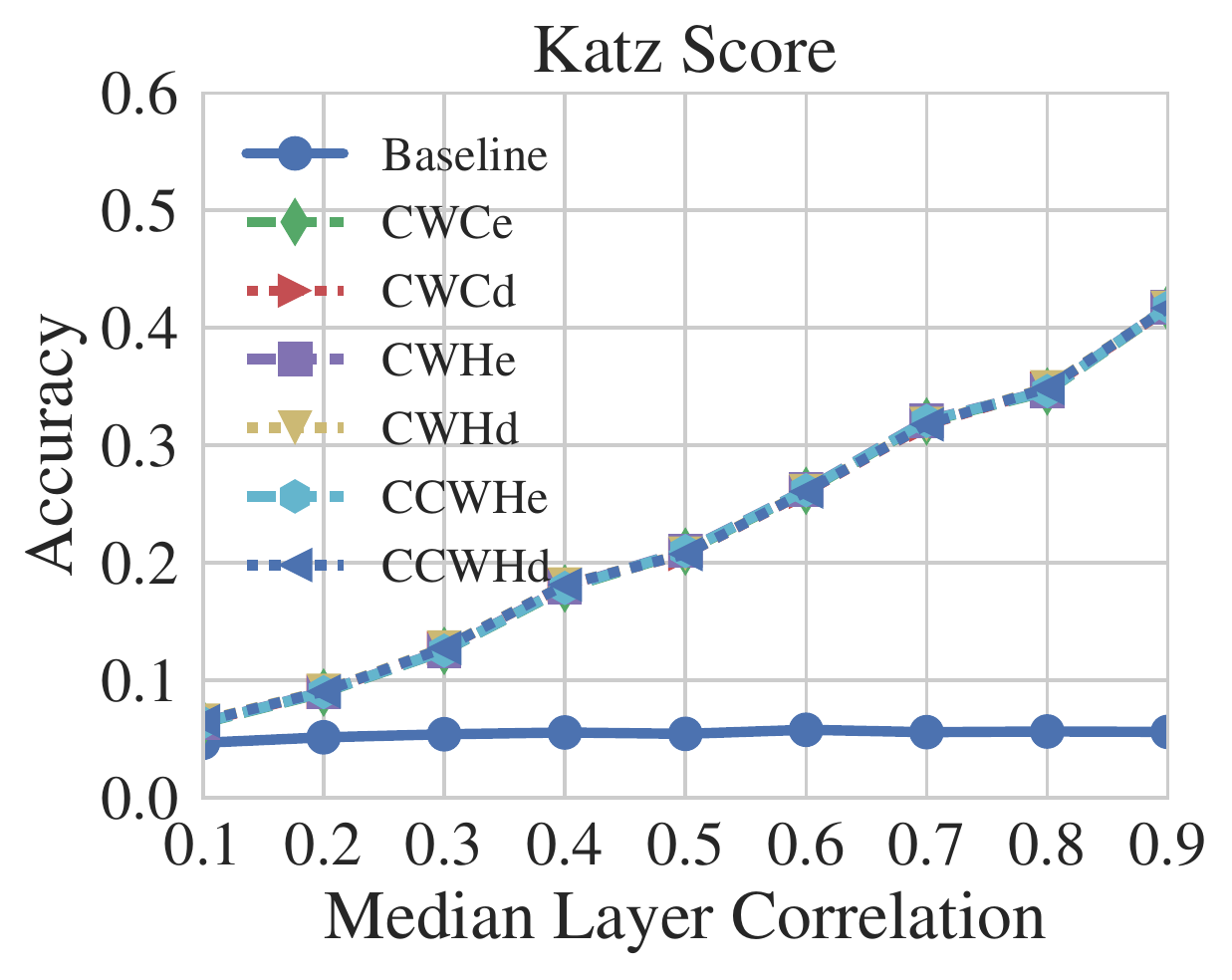}
    \includegraphics[scale=.352]{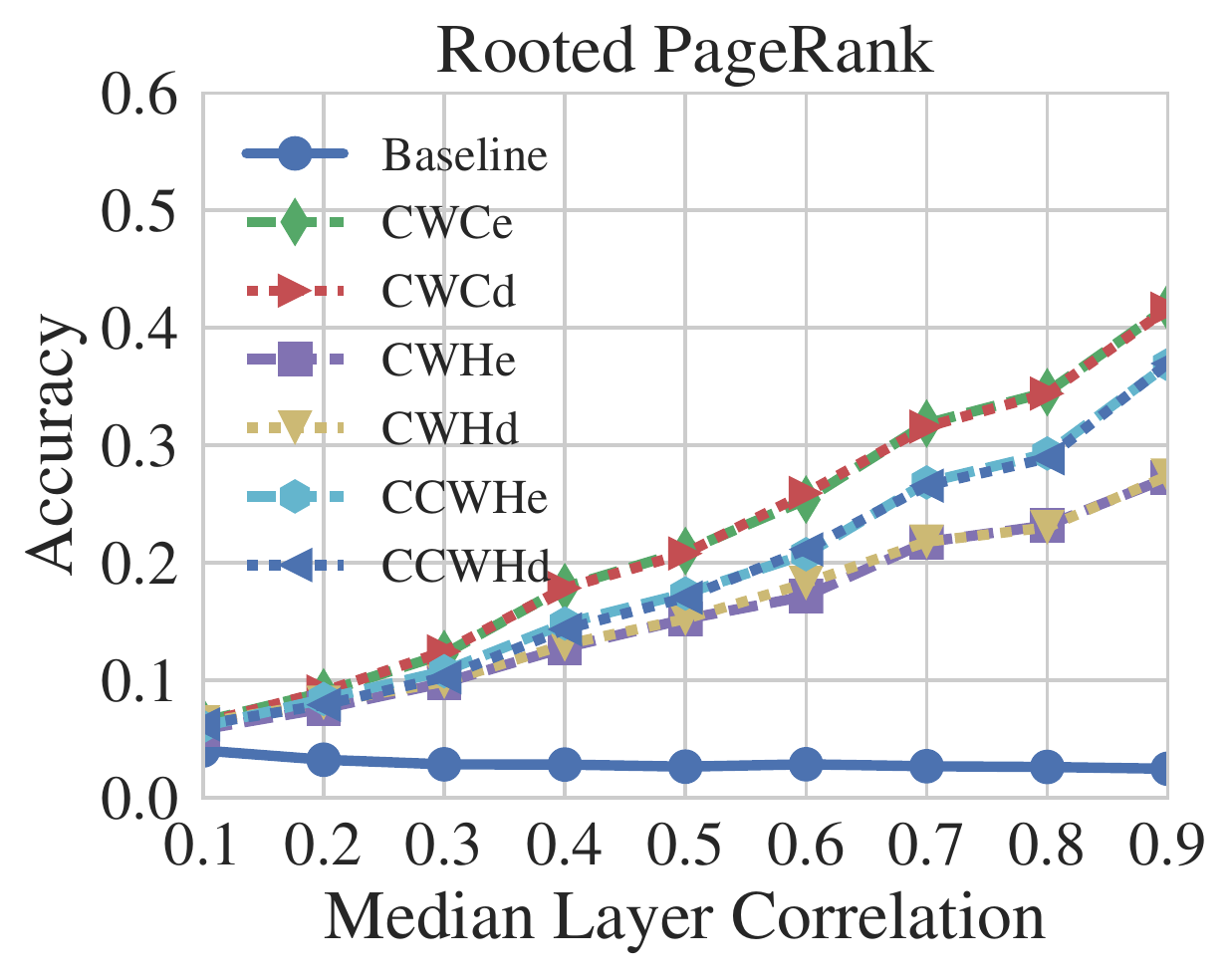}
    
    \caption{Accuracy of the proposed multiplex heuristics and monoplex baselines on synthetic networks with 100 nodes, 10 layers and median cross-layer correlations between 0.10 and 0.90. Larger values indicate higher accuracy.}
\end{figure*}

We first evaluated each of the multiplex network heuristics proposed in the previous section on synthetically generated multiplex networks with varying numbers of nodes, layers and magnitudes of cross-layer correlation. To generate random multiplex networks, we begin by generating random graphs for each layer using the Barab\'asi-Albert random graph generating model, which incorporates the preferential attachment and ``rich get richer'' properties that characterize many real world networks \cite{barabasi:99}. Then, for each node pair in each layer, we add or remove the corresponding edge according to whether it exists at a randomly chosen layer with a specified probability calibrated to match a desired value for median cross-layer correlation (in terms of shared edges). For each random network, we then downsample the edges at each layer by 25\% and evaluate each of our proposed multiplex heuristics for all node pairs at which no link exists. We predict links for the top $x$ scoring pairs corresponding to the number of edges removed. We do this for each layer and average over 100 random networks the percentage of correctly predicted links (predicted edges that were removed during downsampling), which we report as accuracy. We compare the three proposed heuristics to baselines where we evaluate the corresponding monoplex heuristics at the layer being predicted. We plot accuracy against median cross-layer correlation for synthetic networks with 100 nodes and 10 layers in Figure 2. We append 'e' and 'd' to the abbreviations of multiplex heuristics to indicate the usage of edge or degree property matrices when calculating cross-layer correlations.

We first note that both CWC and CCWH significantly outperform all of the baseline monoplex heuristics when cross-layer correlation structure is present, and this outperformance increases linearly with median cross-layer correlation. For the neighbor-based heuristics, CWC, the simplest heuristic, either performs comparable to or better than CCWH, while for the path-based heuristics, CWC and CCWH perform comparably. This is consistent with the finding in \cite{libennowell-04} that simpler heuristics often outperform more complex heuristics in the single-layer case. Furthermore, while CWC is the simplest of the three heuristics, it also most directly captures the richest source of information available when layers are correlated, i.e. whether the edge exists in a highly correlated layer. Thus, in this context, the outperformance of CWC is not surprising. While CWH also outperforms all of the baseline monplex heuristics (not always significantly) the outperformance does not increase with median cross-layer correlations when neighbor-based heuristics are used. This seems to indicate the heuristics applied at additional layers provides limited value when not combined with additional layer specific information, even when cross-layer correlations are significant. However, when path-based heuristics are used, the performance of CWH does increase with median cross-layer correlation indicating the path-based heuristics do pick up on increasingly useful information as cross-layer correlations increase. We observe similar performance when we vary the nodes between 10 and 100 and layers between 5 and 50. In general there is a slight performance increase with more layers, but the increase is minimal once median layer correlation reaches approximately 0.50.

\begin{figure*}
    \centering
    \includegraphics[scale=.256]{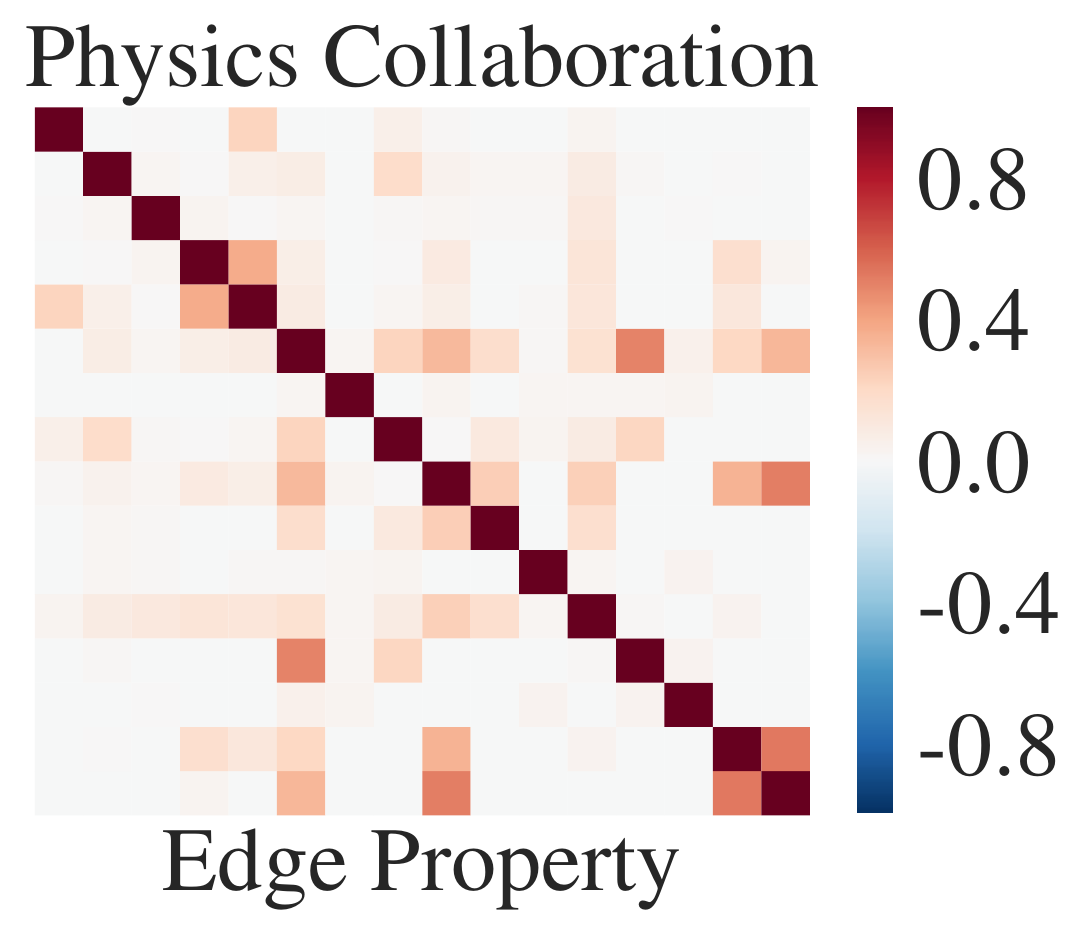}
    \includegraphics[scale=.256]{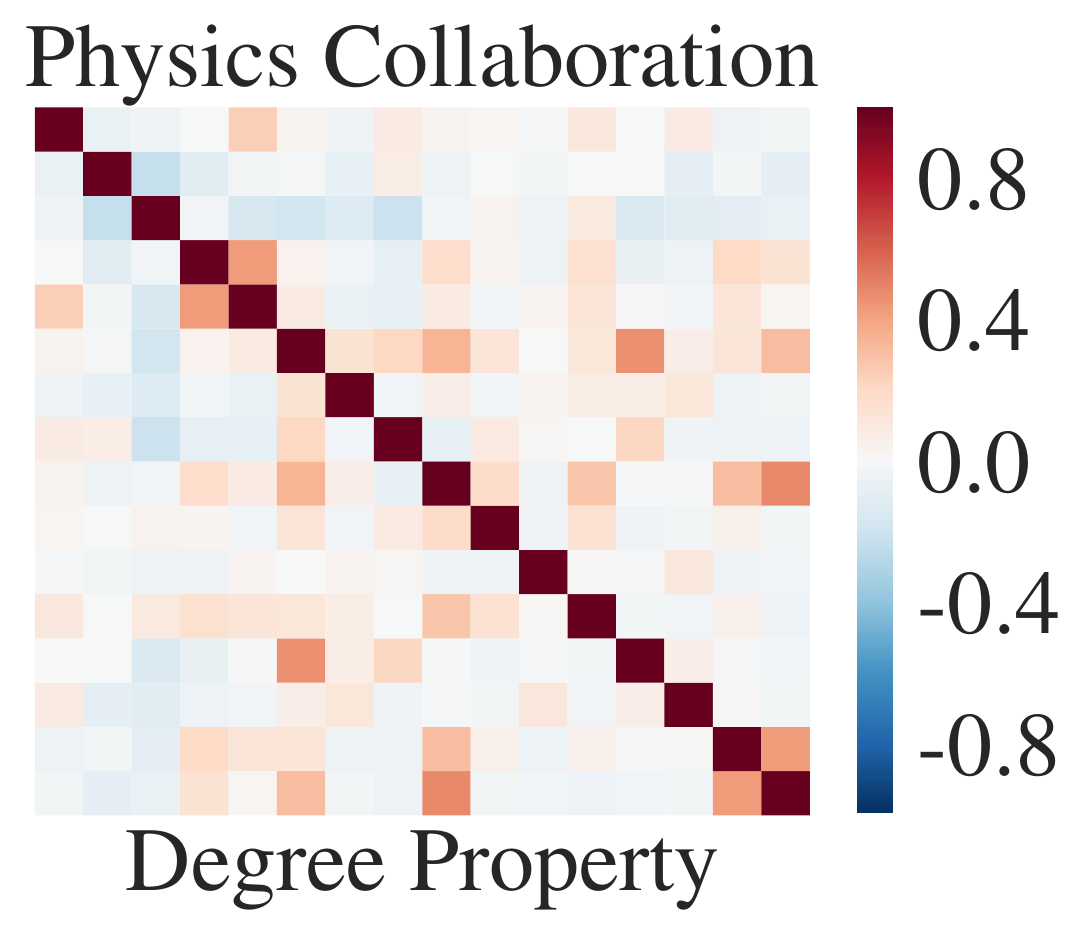}
    \includegraphics[scale=.256]{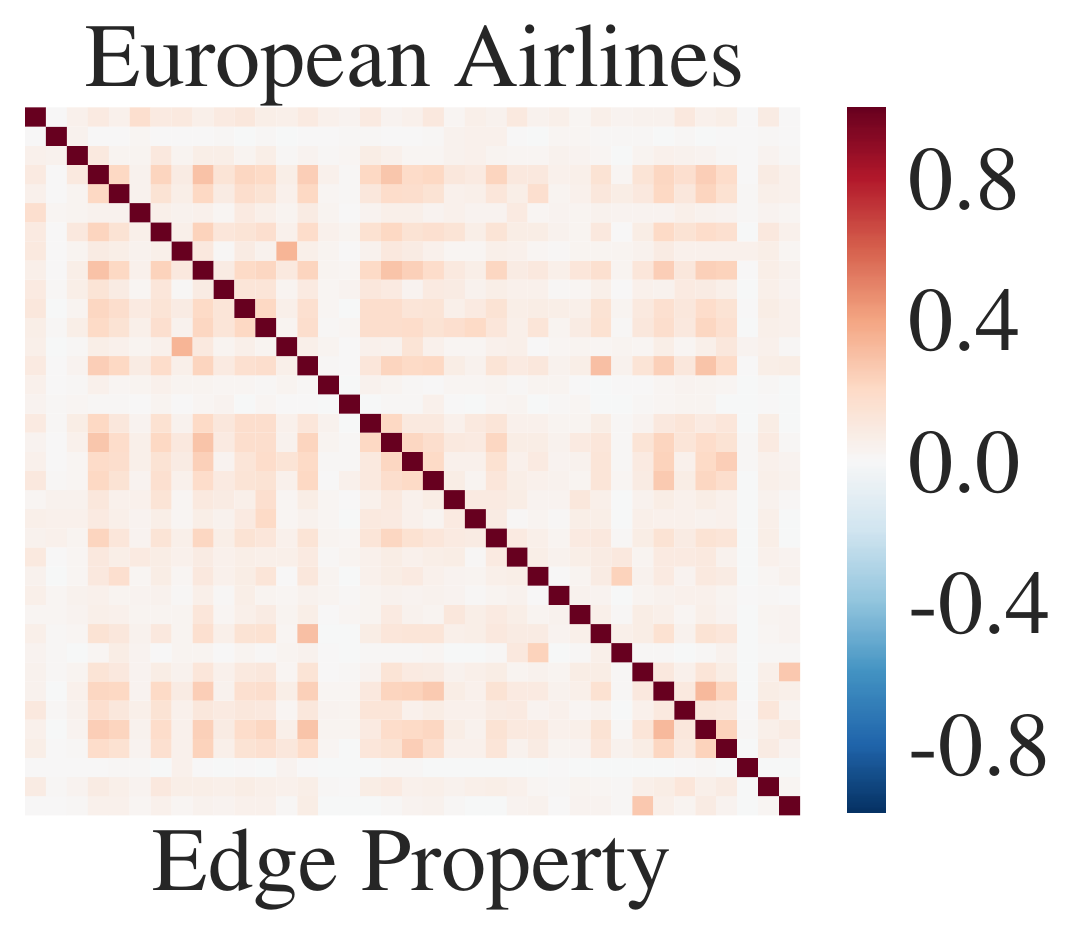}
    \includegraphics[scale=.256]{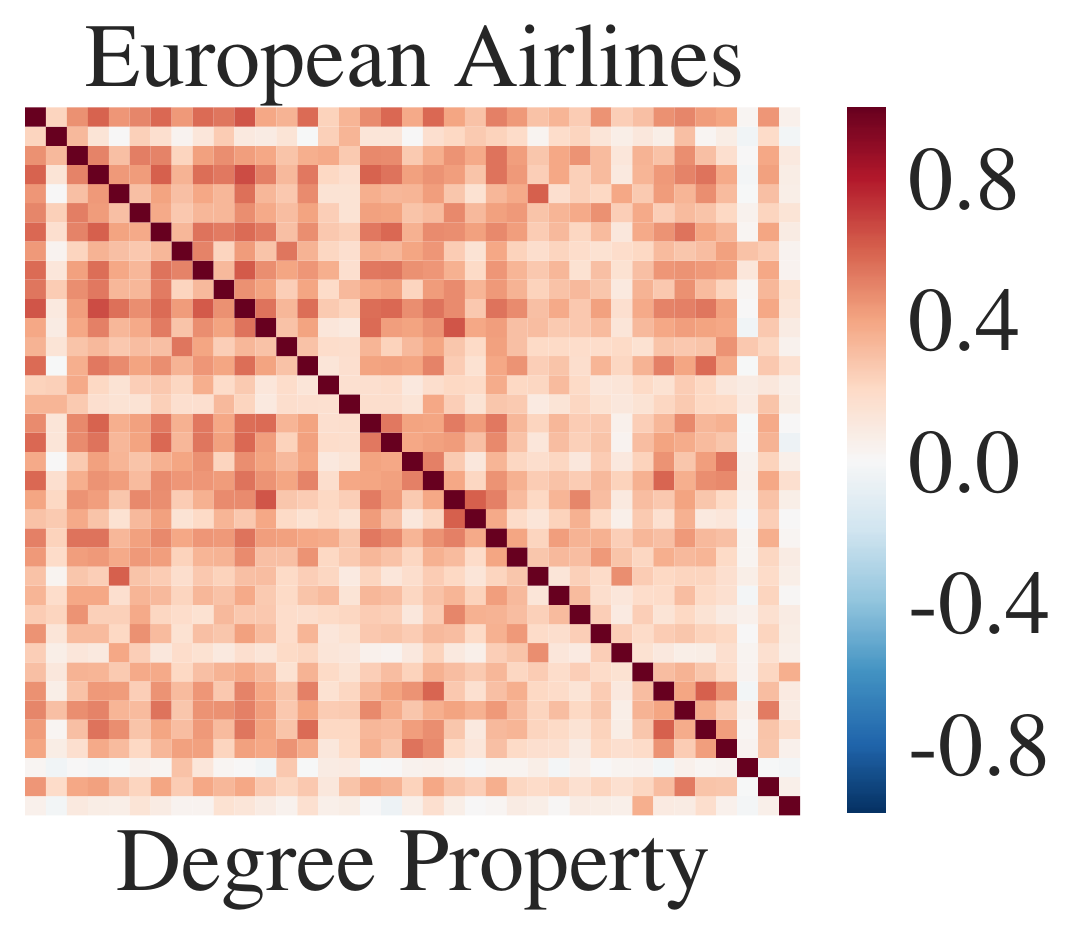}
    \includegraphics[scale=.256]{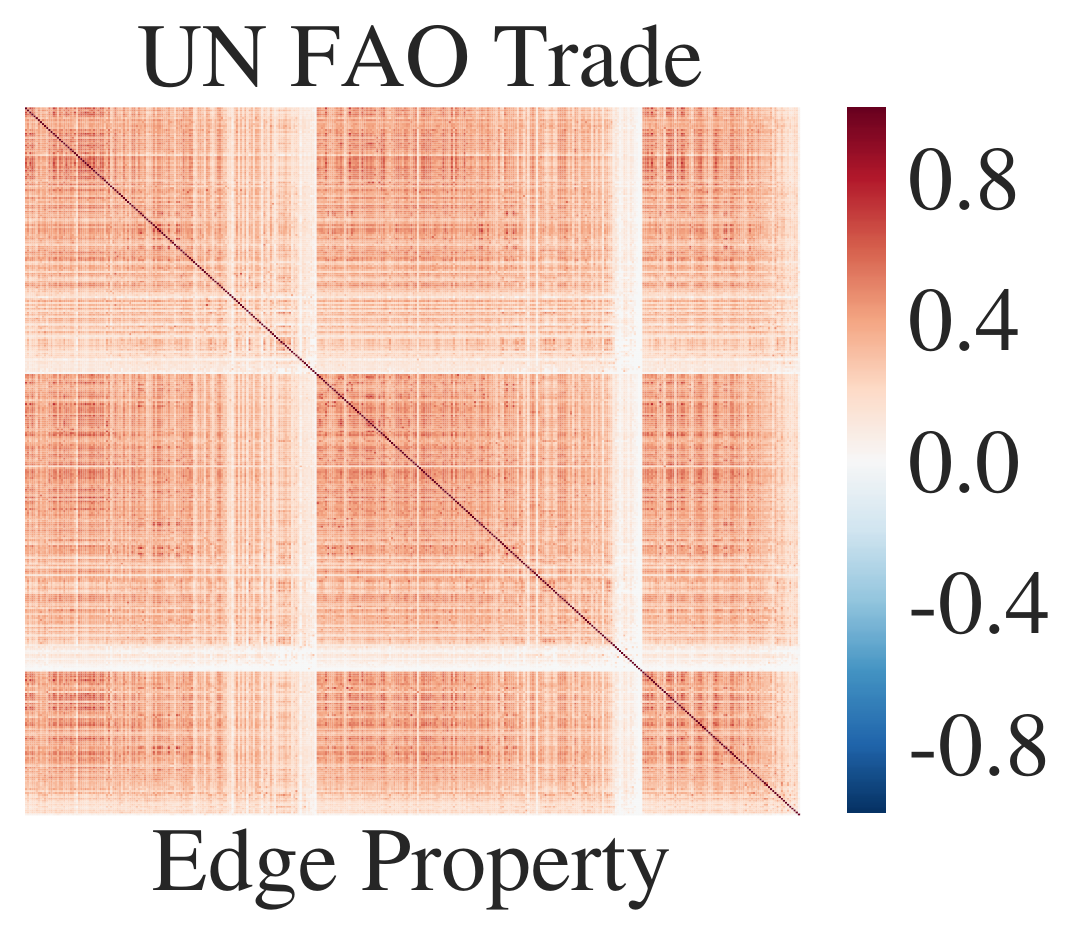}
    \includegraphics[scale=.256]{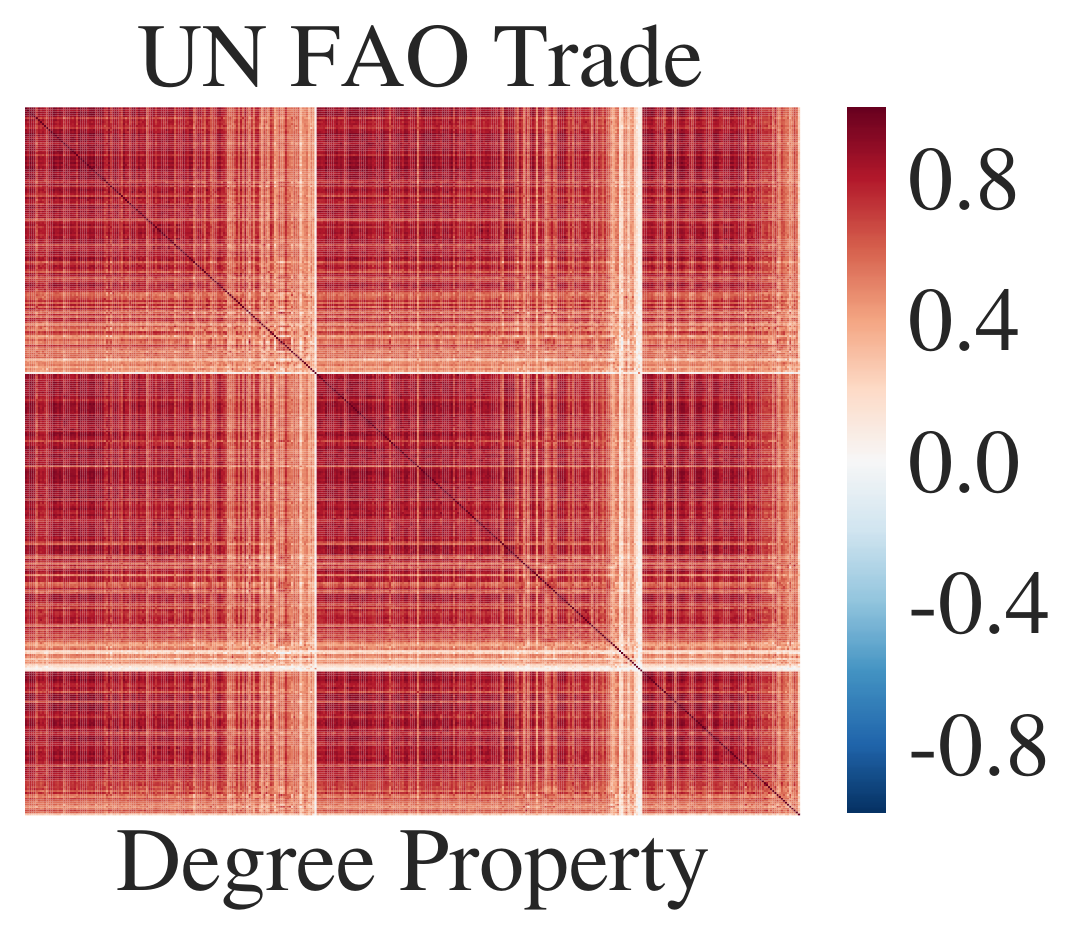}
    \caption{Cross-layer correlation matrices for Pierre Auger Physics Collaboration, European Airlines and UN FAO Global Trade multiplex networks. Darker red cells indicate stronger positive correlations whereas darker blue cells indicate stronger negative correlations.}
\end{figure*}
\begin{figure*}
    \centering
    \includegraphics[scale=.4482]{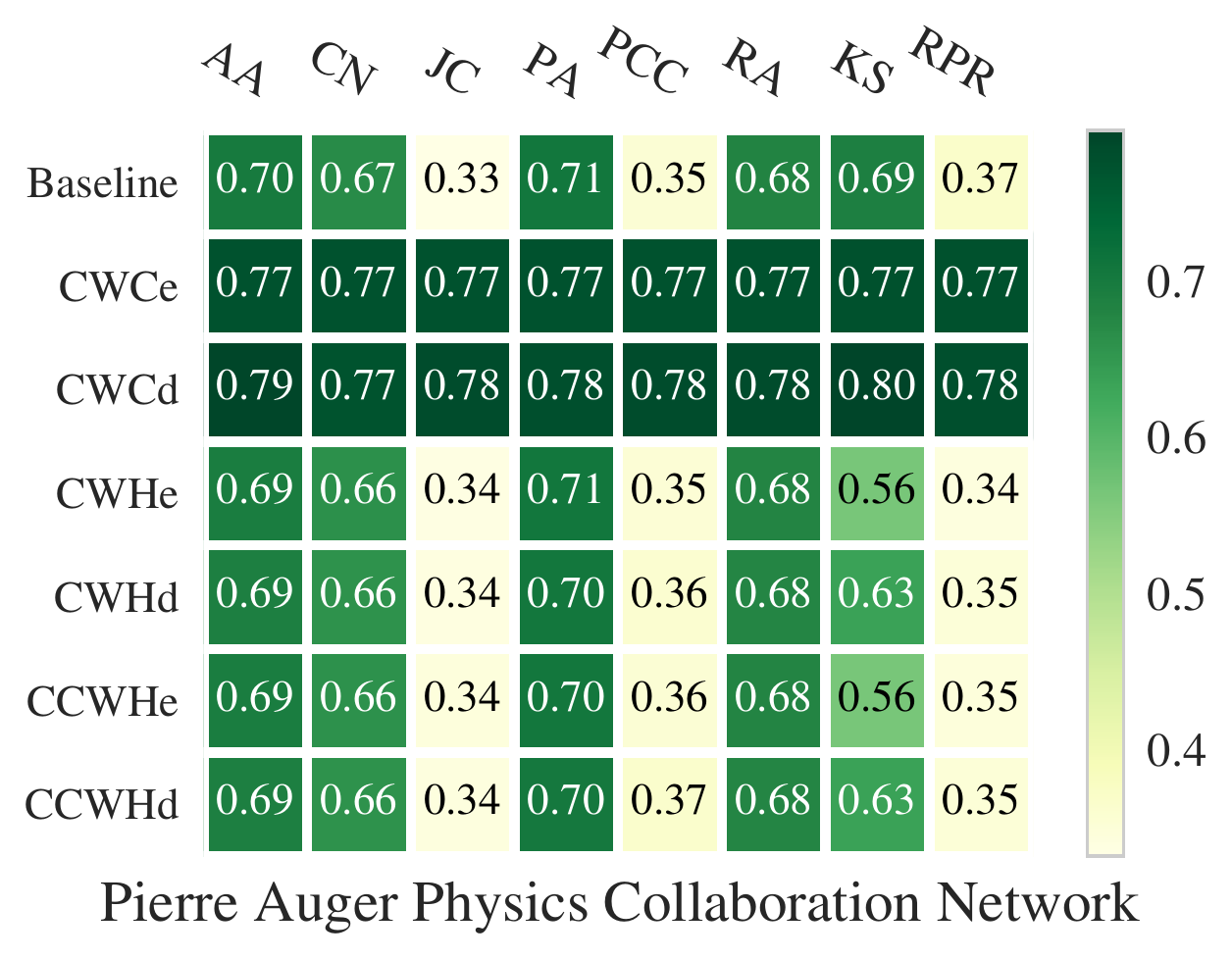}
    \includegraphics[scale=.4482]{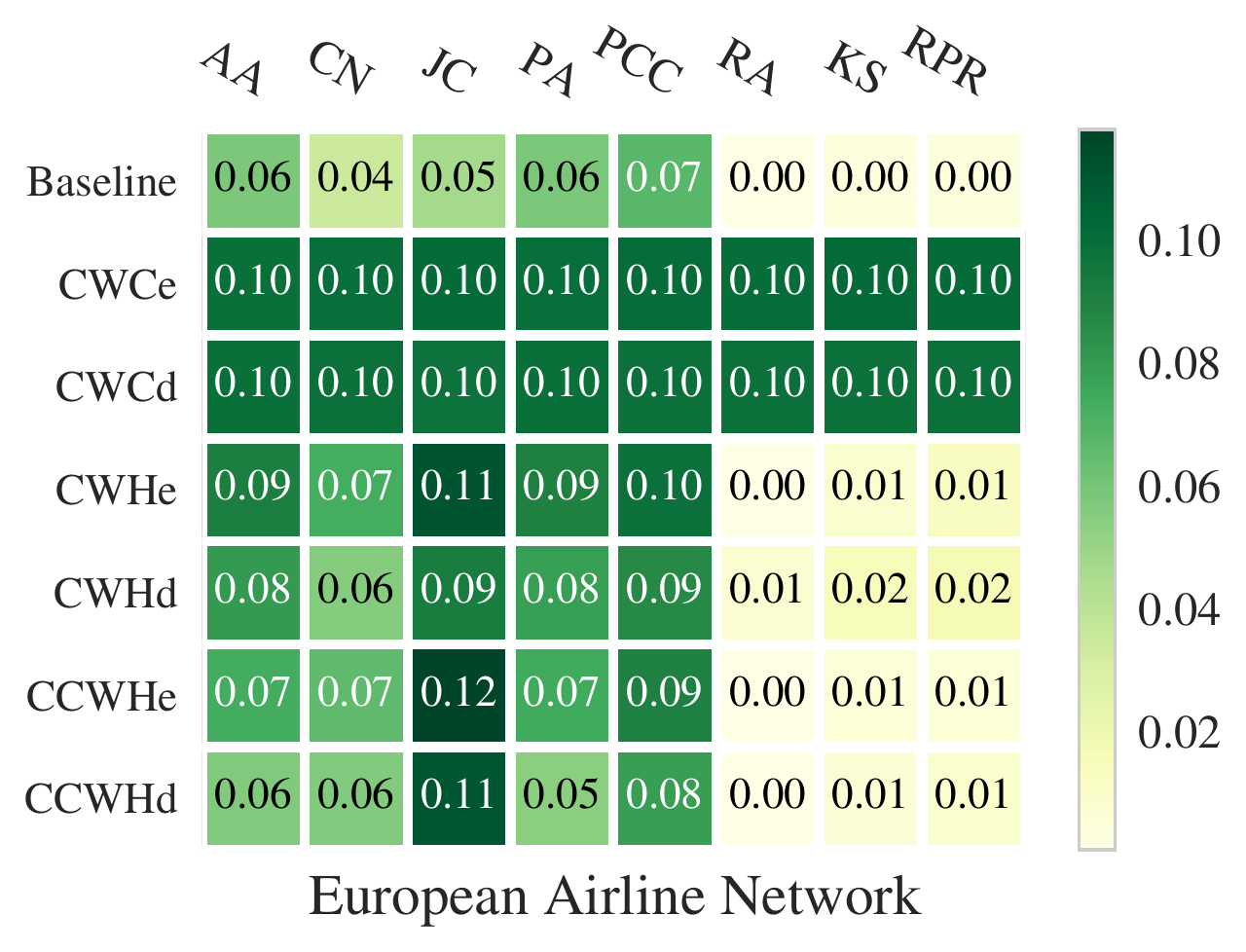}
    \includegraphics[scale=.4482]{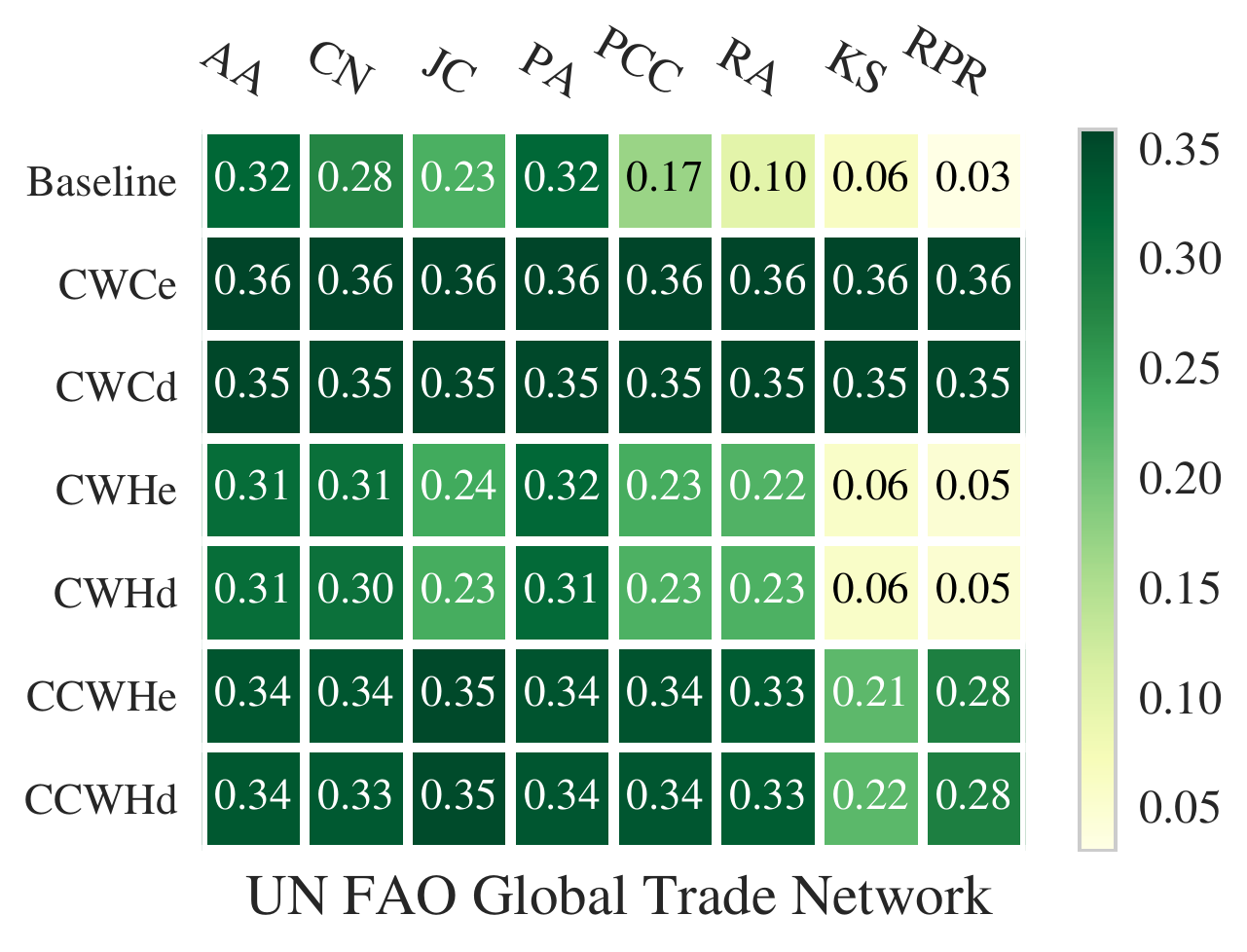}
    \caption{Accuracy of the proposed multiplex heuristics and monoplex baselines on real world scientific collaboration, transportation and global trade multiplex networks. Larger values / darker cells within a specific column indicate higher accuracy.}
\end{figure*}We also evaluated the proposed heuristics on three real world multiplex networks using the same procedure where we downsample edges: a scientific collaboration network with 16 layers representing collaboration on different tasks among 514 scientists at the Pierre Auger Observatory, the largest observatory of ultra-high-energy cosmic rays \cite{domenico:15}, an airline transportation network with 37 layers representing different European airline carriers' direct routes between 450 airports \cite{cardillo:13} and an economic global trade network from the United Nations Food and Agriculture Organization with 364 layers representing import/export relations for a particular food item among 214 countries \cite{domenico:15b}. We show the cross-layer correlation matrices for the edge and degree property matrices in Figure 3, which indicate strong correlation structure, particularly in the case of the UN FAO trade network. Given this strong correlation structure, we should expect the multiplex network heuristics to outperform their monoplex heuristic baselines. For each network, we provide accuracy as a heat map comparing the corresponding monoplex heuristic baselines to each of the proposed multiplex heuristics as columns in Figure 4.

We first note that CWC significantly outperforms all of the monoplex heuristic baselines on all of the real-world networks, consistent with the performance seen in the simulations. CWH and CCWH also either perform better than or the same as each baseline in the airline and trade networks, while their performance is in general similar to their baselines in the collaboration network. We note, however, that in the collaboration network, the baseline performance is already quite high and this network exhibits the weakest correlation structure of these real world networks. Performance is most consistent with the simulations in the trade network, where we see strong outperformance for all of the multiplex heuristics, with the outperformance most significant for CWC and CCWH. We note that this network contains both the most layers and the richest cross-layer correlation structure, which supports our motivation and objective to develop heuristics that take advantage of this structure when present.

Finally, while our focus was to evaluate the proposed heuristics when used for unsupervised link prediction, we also investigated using them as additional features with supervised approaches. Previous supervised approaches for link prediction in multiplex networks have trained separate classifiers for each layer in the network using monoplex heuristics evaluated at each layer as features (as opposed to only using heuristics evaluated at the layer at which links are being predicted). To investigate whether adding our proposed multiplex heuristics as additional features to this set improves supervised performance, we trained Logistic Regression, Naive Bayes and Random Forest classifiers using three different feature sets: \emph{Monoplex-only}, which includes all of the monoplex heuristics discussed in section 2 evaluated at all of the layers in the network, \emph{Multiplex-only}, which includes only the proposed multiplex heuristics CWC, CWH and CCWH using edge and degree cross-layer correlations for each of the monoplex heuristics discussed in section 2 as inputs, and \emph{All features}, which includes both of these sets. To generate these feature sets, we evaluate the heuristics for all node pairs at each layer in the network and make the corresponding label 1 or 0 depending on whether an edge exists. This results in significantly fewer 1 labels so we balance the datasets by subsampling the 0 labeled examples. We then split the datasets into 20\% test data and 80\% training data. We did this for the European Airlines Network and Pierre Auger Physics Collaboration Network, excluding the UN FAO Global Trade Network (the network with the most layers) for computational reasons. We report area under the ROC curve (AUROC) on the test data averaged across the classifiers trained on each layer for each feature set and classifier in tables 1 and 2.
\begin{table}
\caption{AUROC for European Airlines Network} 
\centering 
\begin{tabular}{c | c c c} 
& Monoplex-only & Multiplex-only & All features \\ [0.1ex] 
\hline\vspace{-.25cm}\\ 
Logistic Reg. & 0.893 & 0.957 & 0.900\\
Naive Bayes & 0.963 & 0.954 & 0.977\\
Random Forest & 0.985 & 0.967 & 0.994\\ [.1ex] 
\end{tabular}
\label{table:nonlin} 
\end{table}
\begin{table}
\caption{AUROC for Pierre Auger Physics Collaboration Network} 
\centering 
\begin{tabular}{c | c c c} 
& Monoplex-only & Multiplex-only & All features \\ [0.1ex] 
\hline\vspace{-.25cm}\\ 
Logistic Reg. & 0.995 & 0.999  & 0.996\\
Naive Bayes & 0.995 & 0.998 & 0.999\\
Random Forest & 0.991 & 0.995 &  0.996
\end{tabular}
\label{table:nonlin} 
\end{table}

We first we note that adding the multiplex heuristics to the monoplex features improves performance in terms of both AUROC in all cases (for all classifiers and networks). Additionally, using only the multiplex heuristics leads to greater performance than using all of the monoplex features from all layers when Logistic Regression is used for the European Airlines network and for all of the classifiers trained on the Pierre Auger Physics Collaboration Network. This is despite the fact that this is a much smaller feature set than using the monoplex heuristics evaluated across all layers. While our focus was to provide simple, interpretable heuristics for unsupervised prediction, akin to the similiarity heuristics for unsupervised prediction in monoplex networks, rather than to develop supervised methods, these result provide evidence that our heuristics both (i) add value as additional unique features and (ii) are efficient in that they result in similar or better performance than higher-dimensional feature sets resulting from applying monoplex heuristics across all network layers.

We also note that these AUROC scores are indicative of greater prediction accuracy than those reported in the unsupervised experiments (for both multiplex and monoplex features). This is not simply a consequence of using a supervised method compared to an unsupervised method, but also reflective of the fact that picking a top $k$ ranking of most likely links after a sufficient amount of the network has been corrupted by removing edges is a significantly more difficult problem than predicting whether an edge is present from heuristics which are calculated using a fully-uncorrupted network and provided for all existing edges in a training set. The former problem is more reflective of real-world applications.

\section{Conclusion and Future Work}

We proposed a general framework and three families of multiplex network heuristics for link prediction, CWC, CWH and CCWH. While these heuristics improve supervised methods, they provide a simple, interpretable representation that can be used for efficient unsupervised prediction. Our framework is adaptive to a given problem setting and efficiently takes advantage of rich cross-layer correlation structure when present. Experiments using synthetic and real world networks confirm these heuristics significantly outperform their baselines and performance increases with the strength of correlations.

One line of future research is a more structure specific thresholding procedure: while we find cases of multiplex networks with many correlated and uncorrelated layers where the threshold we provide improves performance, in many cases performance is not affected by using the threshold. If we instead used thresholds based on random graph models that are more specific to the observed structure of a given layer, e.g.  Barab\'asi-Albert random graph models if we observe power-law node degree distributions, this might result in a more robust procedure. Deriving thresholds based on Barab\'asi-Albert and other more complex random graph models is, however, much less straightforward. Another line of open research is developing a more robust procedure for simulating random multiplex networks with specified correlation structures. The procedure we use begins with realistic Barab\'asi-Albert random graphs as layers, but after edges are added and removed to create correlated layers, the degree distribution guarantees of Barab\'asi-Albert graphs are no longer valid. A more robust procedure for generating random multiplex networks would guarantee both a specified layer correlation structure in addition to local properties at each of the layers.

\section*{Disclaimer}

This paper was prepared for information purposes by the AI Research Group of JPMorgan Chase \& Co and its affiliates (“J.P. Morgan”), and is not a product of the Research Department of J.P. Morgan.  J.P. Morgan makes no explicit or implied representation and warranty and accepts no liability, for the completeness, accuracy or reliability of information, or the legal, compliance, financial, tax or accounting effects of matters contained herein.  This document is not intended as investment research or investment advice, or a recommendation, offer or solicitation for the purchase or sale of any security, financial instrument, financial product or service, or to be used in any way for evaluating the merits of participating in any transaction.

\bibliographystyle{ecai} 
\bibliography{aaai.bib}
\end{document}